\documentclass{article}
\usepackage[utf8]{inputenc}
\usepackage[margin=1in]{geometry}
\usepackage{amsmath}
\usepackage{amsthm}
\usepackage{graphicx}
\usepackage{amssymb}
\usepackage{wrapfig}
\usepackage{todonotes}
\usepackage[hidelinks]{hyperref}
\usepackage{verbatim}
\usepackage[square,numbers]{natbib}
\usepackage{algorithm}
\usepackage[noend]{algpseudocode}
\usepackage{bbm}
\usepackage{mathtools}
\usepackage{float}
\usepackage{caption}

\newtheorem{theorem}{Theorem}
\newtheorem{proposition}{Proposition}
\newtheorem{definition}{Definition}
\newtheorem{example}{Example}
\newtheorem{lemma}{Lemma}
\newtheorem{fact}{Fact}

\algnewcommand\CommentLine[1]{
     \hspace{2pt}$\triangleright$ \hspace{2pt}\text{#1}
     }
\newcommand{\poisson}{\text{Poisson}}
\newcommand{\Poisson}{\text{Poisson}}
\newcommand{\nmu}{\norm{\hat{\mu}}}
\newcommand{\R}{\mathbb{R}}
\newcommand{\E}{\mathbb{E}}
\newcommand{\Var}{\text{Var}}
\newcommand{\eps}{\epsilon}

\newcommand{\norm}[1]{\| #1 \|}
\newcommand{\ip}[2]{\langle {#1} , {#2} \rangle}

\newcommand{\new}{\text{new}}
\newcommand{\edit}[1]{{\color{red}#1}}
\DeclarePairedDelimiterX{\infdivx}[2]{(}{)}{%
  #1\;\delimsize\|\;#2%
}
\newcommand{\DKL}{D_{KL}\infdivx}

\title{Sample Amplification: \\Increasing Dataset Size even when Learning is Impossible}
\vspace{5pt}
\author{Brian Axelrod \hspace{20pt} Shivam Garg \hspace{20pt} Vatsal Sharan \hspace{20pt} Gregory Valiant \vspace{5pt}\\   \texttt{\{baxelrod, shivamgarg, vsharan, valiant\}@stanford.edu}\vspace{5pt} \\ Stanford University \thanks{This work was supported by NSF awards AF-1813049 and CCF-1704417, an ONR Young Investigator Award N00014-18-1-2295, and an NSF Graduate Fellowship.}}

\date{}

\newcommand{\algdiscrete}{
\begin{algorithm}[H]
  \caption{Sample Amplification for Discrete Distributions} 
  \label{alg:discrete} % label needs to go after caption for numbering to work properly
  \textbf{Input}: $X_{4n} = (x_1, x_2, \dots, x_{4n})$, where $x_i \leftarrow D$, for any discrete distribution $D$ over $[k]$.\\
   \textbf{Output}: $Z_m = (x_1', x_2', \dots, x_m')$, such that $D_{TV}(D^m, Z_m) \leq \frac{1}{3}$.
   \begin{algorithmic}[1]
   
 \Procedure{AmplifyDiscrete}{$X_{4n}$}
\State $N_1, N_2 \leftarrow \poisson(n)$
\Comment{Draw two i.i.d samples $N_1$ and $N_2$ from $\poisson({n})$}
\State $N := N_1 + N_2$

\If {$N \leq {4n}$ }
\State $X_{N_1} := (x_1, x_2, \dots, x_{N_1})$ \State $X_{N_2} := (x_{N_1+1}, x_{N_1+2}, \dots, x_{N_1+N_2})$
\vspace{5pt}
\Else
\Comment{Uninteresting case: happens with low probability}
\State $X_{N_1} := \underbrace{(x_1, x_1, \dots, x_{1})}_{N_1 \text{ times}}$ 
\State $X_{N_2} := \underbrace{(x_1, x_1, \dots, x_{1})}_{N_2 \text{ times}}$\EndIf
\vspace{-5pt}

\State $r := 8(m - n)$
\State $(x_1', x_2', \dots, x_{N+R}')$ = \Call{AmplifyDiscretePoissonized}{$X_{N_1}, X_{N_2}, r, n$}

\State \CommentLine{Amplify first $N_1 + N_2$ samples to $N_1+N_2+R$ samples, for $R$ roughly distributed as $\poisson(r)$}
\vspace{5pt}

\State $R_1 := max(R, r/8)$

\If{$R < r/8$}
\Comment{Uninteresting case: happens with low probability}
\vspace{2 pt}
\State $(x_1', x_2', \dots, x_{N+R_1}') := (x_1', x_2', \dots, x_{N+R}', \underbrace{x_1, x_1, \dots, x_1}_{\frac{r}{8} - R \text{ times}})$
\EndIf

\State $(x_{N+R_1+1}', x_{N+R_1+2}', \dots, x_{m}') := (x_{N+1}, x_{N+2}, \dots, x_{4n - (R_1-\frac{r}{8})})$
\State \CommentLine{Add the remaining samples to get $4n+r/8$ samples in total}
\vspace{5pt}
\State $Z_m := (x_1', x_2', \dots, x_m')$

\State \textbf{return} $Z_M$
 \EndProcedure
 
 \Statex

  \Procedure{AmplifyDiscretePoissonized}{$X_{N_1}, X_{N_2}, r, n$}
\State \CommentLine{Generates approximately $\poisson(r)$ more samples given $N_1+N_2$ input samples}
 \State \CommentLine{$X_{N_1} = (x_1, x_2, \dots, x_{N_1})$, $X_{N_2} = (x_{N_1+1}, x_{N_1+2}, \dots, x_{N_1+N_2})$, and $r = 8(m - n) $}
 \vspace{5pt}
  \State $\text{count}_j := \sum_{i=1}^{N_1} \mathbbm{1}(x_i = j)$, for $j \in [k]$
\Comment{Find the count of each element in first $N_1$ samples}
\vspace{5pt}
\State $\hat{p}_j := \frac{count_j}{n}$, for $j \in [k]$ 
\State $\hat{z}_j \leftarrow \poisson(\hat{p}_j r)$, for $j \in [k]$ 

\State $R:= \sum_{j=1}^k \hat{z}_j$
  
  \State $(x_1', x_2', \dots, x_{N_1}') := (x_1, x_2, \dots, x_{N_1})$
\State $(x_{N_1+1}', \dots, x_{N_1+N_2+R}') := \text{RandomPermute}((x_{N_1+1}, x_{N_1+2}, \dots, x_{N_1+N_2}, \underbrace{1, 1, \dots, 1}_{\hat{z}_1 \text{ times}}, \dots, \underbrace{ k, k, \dots, k}_{\hat{z}_k \text{ times}} ))$
 \State \textbf{return} $(x_1', x_2', \dots, x_{N_1+N_2+R}')$
  \EndProcedure
 
  \end{algorithmic}
\end{algorithm}
}

\newcommand{\alggaussiannonmod}{
\begin{algorithm*}
  \caption{Sample Amplification for Gaussian with Unknown Mean and Fixed Covariance Without Modifying Input Samples} 
  \label{alg:gaussian2}
  \textbf{Input}: $X_n = (x_1, x_2, \dots, x_n)$, where $x_i \leftarrow N(\mu, \Sigma_{d \times d})$.\\
   \textbf{Output}: $Z_m = (x_1', x_2', \dots, x_m')$, such that $D_{TV}(D^m, Z_m) \leq \frac{1}{3}$, where $D$ is $N(\mu, \Sigma_{d \times d})$
  \begin{algorithmic}[1]
 \Procedure{AmplifyGaussian2}{$X_n$}
 \State $r := m - n$
 \State  $\hat{\mu} := \sum_{i=1}^{\frac{n}{2}} \frac{x_i}{n/2}$
 \State $x_i' := x_i  $, for $i \in \{1, 2, \dots, \frac{n}{2}\}$
 \State $X_{\text{remaining}} := (x_{\frac{n}{2}+1}, x_{\frac{n}{2}+2}, \dots, x_n)$

 \For{$i=\frac{n}{2}+1$ \text{\textbf{ to }} m}
 \State $T \leftarrow \text{Bernoulli}(\frac{2r}{r+n/2}) $
 \Comment{Set $T = 1$ with probability $\frac{2r}{r+n/2}$, and $0$ otherwise}
 \If {$T \text{ equals } 1$}
 \State $x_i' \leftarrow N(\hat{\mu}, \Sigma_{d \times d})$
 \Else 
 \If{$X_{\text{remaining}}$ is not empty}
 \State $x_i' :=$ Random Element Drawn without Replacement from $X_{\text{remaining}}$
 \Else
 \State $x_i' := x_1$
 \Comment Happens with small probability
 \EndIf
 \EndIf
 \EndFor
 \State $Z_m := (x_1', x_2', \dots, x_m')$
 \State \textbf{return} $Z_m$
 \EndProcedure
  \end{algorithmic}
\end{algorithm*}
}

\newcommand{\alggaussian}{
\begin{algorithm*}
  \caption{Sample Amplification for Gaussian with Unknown Mean and Fixed Covariance} 
    \label{alg:gaussian1}
  \textbf{Input}: $X_n = (x_1, x_2, \dots, x_n)$, where $x_i \leftarrow N(\mu, \Sigma_{d \times d})$.\\
   \textbf{Output}: $Z_m = (x_1', x_2', \dots, x_m')$, such that $D_{TV}(D^m, Z_m) \leq \frac{1}{3}$, where $D$ is $N(\mu, \Sigma_{d \times d})$
  \begin{algorithmic}[1]
 \Procedure{AmplifyGaussian}{$X_n$}
 
 \State  $\hat{\mu} := \sum_{i=1}^{n} \frac{x_i}{n}$
 \State $\epsilon_i \leftarrow N(0, \Sigma_{d \times d})$, for  $i \in \{n+1, n+2, \dots, m\}$
 \Comment{Draw $m-n$ i.i.d samples from $N(0, \Sigma_{d \times d})$}
 \State $x_i' := \hat{\mu} + \epsilon_i$, for $i \in \{n+1, n+2, \dots, m\}$
 \State $x_i' := x_i - \sum_{j=n+1}^m{\frac{\epsilon_j}{n}} $, for $i \in \{1, 2, \dots, n\}$
 \Comment{Remove correlations between old and new samples}
 \State \textbf{return} $Z_m:= (x_1', x_2', \dots, x_m')$
 \EndProcedure
  \end{algorithmic}
\end{algorithm*}
}

\begin{document}

\maketitle
\thispagestyle{empty}
\begin{abstract}
Given data drawn from an unknown distribution, $D$, to what extent is it possible to ``amplify'' this dataset and faithfully output an even larger set of samples that appear to have been drawn from $D$?  We formalize this question as follows:  an $(n,m)$ \emph{amplification procedure} takes as input $n$ independent draws from an unknown distribution $D$, and outputs a set of $m > n$ ``samples''.  An amplification procedure is \emph{valid} if no algorithm can distinguish the set of $m$ samples produced by the amplifier from a set of $m$ independent draws from $D$, with probability greater than $2/3$.  Perhaps surprisingly, in many settings, a valid amplification procedure exists, even in the regime where the size of the input dataset, $n$, is significantly less than what would be necessary to \emph{learn} distribution $D$ to non-trivial accuracy.  Specifically we consider two fundamental settings: the case where $D$ is an arbitrary discrete distribution supported on $\le k$ elements, and the case where $D$ is a $d$-dimensional Gaussian with unknown mean, and fixed covariance matrix.  In the first case, we show that an $\left(n, n + \Theta(\frac{n}{\sqrt{k}})\right)$ amplifier exists.   In particular, given $n=O(\sqrt{k})$ samples from $D$, one can output a set of $m=n+1$ datapoints, whose total variation distance from the distribution of $m$ i.i.d. draws from $D$ is a small constant, despite the fact that one would need quadratically more data, $n=\Theta(k)$, to \emph{learn} $D$ up to small constant total variation distance. In the Gaussian case, we show that an $\left(n,n+\Theta(\frac{n}{\sqrt{d}} )\right)$ amplifier exists, even though learning the distribution to small constant total variation distance requires $\Theta(d)$ samples.  In both the discrete  and Gaussian settings, we show that these results are tight, to constant factors.  Beyond these results, we describe potential applications of such data amplification, and formalize a number of curious directions for future research along this vein. 
\end{abstract}

\newpage
\setcounter{page}{1}
\section{Learning, Testing, and Sample Amplification}

How much do you need to know about a distribution, $D$, in order to produce a dataset of size $m$ that is indistinguishable from a set of independent draws from $D$?  Do you need to \emph{learn} $D$, to nontrivial accuracy in some natural metric, or does it suffice to have access to a smaller dataset of size $n<m$ drawn from $D$, and then ``amplify'' this dataset to create one of size $m$?  In this work we formalize this question, and show that for two natural classes of distribution, discrete distributions with bounded support, and $d$-dimensional Gaussians, non-trivial data ``amplification'' is possible even in the regime in which you are given too few samples to learn.

From a theoretical perspective, this question is related to the meta-question underlying work on distributional property testing and estimation: \emph{To answer basic hypothesis testing or property estimation questions regarding a distribution $D$, to what extent must one first learn $D$, and can such questions be reliably answered given a relatively modest amount of data drawn from $D$?}  Much of the excitement surrounding distributional property testing and estimation stems from the fact that, for many such testing and estimation questions, a surprisingly small set of samples from $D$ suffices---significantly fewer samples than would be required to learn $D$.  These surprising answers have been revealed over the past two decades.  The question posed in our work fits with this body of work, though instead of asking how much data is required to perform a hypothesis test, we are asking how much data is required to \emph{fool} an optimal hypothesis test---in this case an ``identity tester'' which knows $D$ and is trying to distinguish a set of $m$ independent samples drawn from $D$, versus $m$ datapoints constructed in some other fashion.

From a more practical perspective, the question we consider also seems timely.  Deep neural network based systems, trained on a set of samples, can be designed to perform many tasks, including testing whether a given input was drawn from a distribution in question (i.e. ``discrimination''), as well as sampling (often via the popular Generative Adversarial Network (GAN) approach).  There are many relevant questions regarding the extent to which current systems are successful  in accomplishing these tasks, and the question of how to quantify the performance of these systems is still largely open.   In this work, however, we ask a different question: Suppose a system \emph{can} accomplish such a task---what would that actually mean?  If a system can produce a dataset that is indistinguishable from a set of $m$ independent draws from a distribution, $D$, does that mean the system knows $D$, or are there other ways of accomplishing this task?

\subsection{Formal Problem Definition}

We begin by formally stating two essentially equivalent definitions of sample amplification and then provide an illustrative example.  Our first definition states that a function $f$ mapping a set of $n$ datapoints to a set of $m$ datapoints is a valid amplification procedure for a class of distributions $\mathcal{C}$, if for all $D \in \mathcal{C}$, letting $X_n$ denote the random variable corresponding to $n$ independent draws from $D$, the distribution of $f(X_n)$ has small total variation distance\footnote{We overload the notation  $D_{TV}(\cdot, \cdot)$ for total variation distance, and also use it when the argument is a random variable instead of the distribution of the random variable, whenever convenient.} to the distribution defined by $m$ independent draws from $D$.  

%\footnote{We abuse the notation for total variation distance, and define it between two random variables, or a random variable and a distribution, instead of between the associated distributions, whenever convenient.} 

\begin{definition}\label{def1}
A class $\mathcal{C}$ of distributions over domain $S$ admits an $(n,m)$ \emph{amplification procedure} if there exists a (possibly randomized) function $f_{\mathcal{C},n,m} : S^n \rightarrow S^m$, mapping a dataset of size $n$ to a dataset of size $m$, such that for every distribution $D \in \mathcal{C}$, $$D_{TV}\left(f_{\mathcal{C},n,m}(X_n), D^{m} \right) \le 1/3,$$ where $X_n$ is the random variable denoting $n$ independent draws from $D$, and $D^m$ denotes the distribution of $m$ independent draws from $D$.   If no such function $f_{\mathcal{C},n,m}$ exists, we say that $\mathcal{C}$ \emph{does not admit} an $(n,m)$ amplification scheme.

\end{definition}

Crucially, in the above definition we are considering the random variable $f(X_n)$ whose randomness comes from the randomness of $X_n$, as well as any randomness in the function $f$ itself.  For example, every class of distributions admits an $(n,n)$ amplification procedure, corresponding to taking the function $f$ to be the identity function.   If, instead, our definition had required that the \emph{conditional} distribution of $f(X_n)$ given $X_n$ be close to $D^m$, then the above definition would simply correspond to asking how well we can \emph{learn} $D$, given the $n$ samples denoted by $X_n$.

Definition~\ref{def1} is also equivalent, up to the choice of constant $1/3$ in the bound on total variation distance, to the following intuitive formulation of sample amplification as a game between two parties: the ``amplifier'' who will produce a dataset of size $m$, and a ``verifier'' who knows $D$ and will either accept or reject that dataset.   The verifier's protocol, however, must satisfy the condition that given $m$ independent draws from  the true distribution in question, the verifier must accept with probability at least $3/4$, where the probability is with respect to both the randomness of the set of samples, and any internal randomness of the verifier. We briefly describe this formulation, as it parallels the pseudo-randomness framework, and a number of natural directions for future work---such as if the verifier is computationally bounded, or only has sample access to $D$---are easier to articulate in this setting.

\begin{definition}\label{def2}
The \emph{sample amplification} game consists of two parties, an \emph{amplifier} corresponding to a function $f_{n,m}: S^n \rightarrow S^m$ which maps a set of $n$ datapoints in domain $S$ to a set of $m$ datapoints, and a \emph{verifier} corresponding to a function $v:S^m \rightarrow \{ACCEPT, REJECT\}$.  We say that a verifier $v$ is \emph{valid} for distribution $D$ if, when given as input a set of $m$ independent draws from $D$, the verifier accepts with probability at least $3/4$, where the probability is over both the randomness of the draws and any internal randomness of $v$: $$\Pr_{X_m \leftarrow D^m}[v(X_m) = ACCEPT] \ge 3/4.$$   A class $\mathcal{C}$ of distributions over domain $S$ admits an $(n,m)$ \emph{amplification procedure} if, and only if, there is an amplifier function $f_{\mathcal{C},n,m}$ that, for every $D\in \mathcal{C}$, can ``win'' the game with probability at least $2/3$; namely, such that for every $D \in \mathcal{C}$ and valid verifier $v_D$ for $D$ $$\Pr_{X_n \leftarrow D^n}[v_D(f_{\mathcal{C},n,m}(X_n))=ACCEPT] \ge 2/3,$$ where the probability is with respect to the randomness of the choice of the $n$ samples, $X_n,$ and any internal randomness in the amplifier and verifier, $f$ and $v$.
\end{definition} 

As was the case in Definition~\ref{def1}, in the above definition it is essential that the verifier only observes the output $f(X_n)$ produced by the amplifier.  If the verifier sees both the amplified samples, $f(X_n)$ in addition to the original data, $X_n$, then the above definition also becomes equivalent to asking how well the class of distributions in question can be \emph{learned} given $n$ samples. \\

\vspace{-10pt}
\begin{figure}[h]
\centering
\includegraphics[scale=0.45]{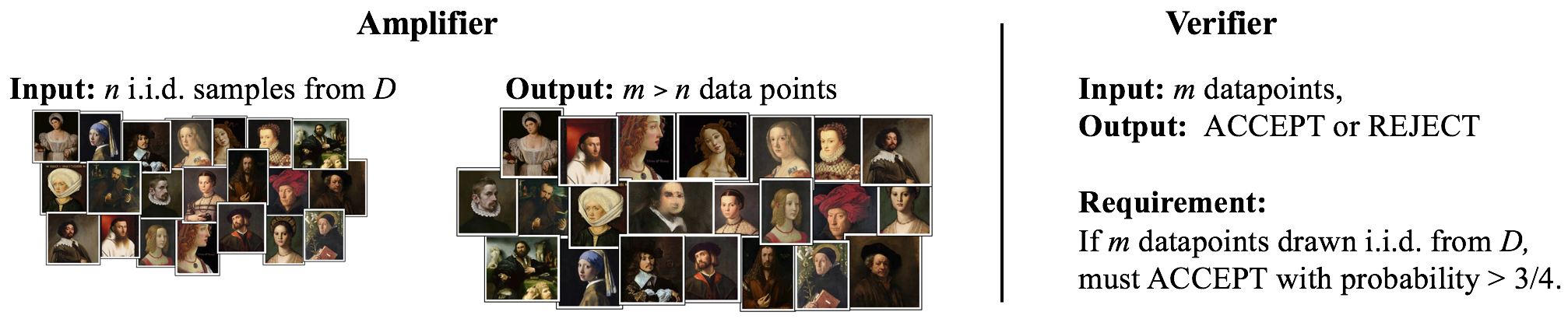}
\caption{Sample amplification can be viewed as a game between an ``amplifier'' that obtains $n$ independent draws from an unknown distribution $D$ and must output a set of $m > n$ samples, and a  ``verifier'' that receives the $m$ samples and must ACCEPT or REJECT.  The verifier knows the true distribution $D$ and is computationally unbounded but does not know the amplifier's training set (the set of $n$ input samples).  An amplification scheme is successful if, for every verifier, with probability at least $2/3$ the verifier will accept the output of the amplifier.  [In the setting illustrated above, observant readers might recognize that one of the images in the ``Output'' set is a painting which was sold in October, 2018 for over \$400k by Christie's auction house, and which was ``painted'' by a Generative Adversarial Network (GAN)~\cite{cohn_2018}].}
\end{figure}

%description of the differences between the two definitions---as the constants change

\vspace{-15pt}
\begin{example}
Consider the class of distributions $\mathcal{C}$ corresponding to i.i.d. flips of a coin with unknown bias $p$.  We claim that there are constants $c' \ge c > 0$ such that $(n, n+cn)$ sample amplification is possible, but $(n, n+c'n)$ amplification is not possible.  To see this, consider the amplification strategy corresponding to returning a random permutation of the original samples together with $cn$ additional tosses of a coin with bias $\hat{p}$, where $\hat{p}$ is the empirical bias of the $n$ original samples.  Because of the random permutation, the total variation distance between these samples and $n+cn$ i.i.d. tosses of the $p$-biased coin is a function of only the distribution of the total number of $heads$.  Hence this is equivalent to the distance between $Binomial(n+cn,p),$ and the distribution corresponding to first drawing $h \leftarrow Binomial(n,p)$, and then returning $h+Binomial(cn,h/n)$.  It is not hard to show that the total variation distance between these two can be bounded by any small constant by taking $c$ to be a sufficiently small constant.  Intuitively, this is because both distributions have the same mean, they are both unimodal, and have variances that differ by a small constant factor for small constant $c$.   For the lower bound, to see that amplification by more than a constant factor is impossible, note that if it were possible, then one could learn $p$ to error $o(1/\sqrt{n})$, with small constant probability of failure, by first amplifying the original samples and then returning the empirical estimate of $p$ based on the amplified samples. 

In the above setting, this constant factor amplification is not surprising, since the amplifier can learn the distribution to non-trivial accuracy.  It is worth observing, however, that the above amplification scheme corresponding to a $(n, n+1)$ amplifier will return a set of $n+1$ samples, whose total variation distance from $n+1$ i.i.d. samples is only $O(1/n)$;  this is despite the fact that the amplifier can only learn the distribution to total variation distance $\Theta(1/\sqrt{n}).$
\end{example}

\vspace{-10pt}
\subsection{Summary of Results}

Our main results provide tight bounds on the extent to which sample amplification is possible for two fundamental settings, unstructured discrete distributions, and $d$-dimensional Gaussians with unknown mean and fixed covariance.  Our first result is for discrete distributions with support size at most $k$. In this case, we show that sample amplification is possible given only $O(\sqrt{k})$ samples from the distribution, and tightly characterize the extent to which amplification is possible.\footnote{This addresses a variant of an open problem posed in the Frontiers in Distribution Testing workshop at FOCS 2017 (\url{https://sublinear.info/index.php?title=Open_Problems:85}).} Note that learning the distribution to small total variation distance requires $\Theta(k)$ samples in this case.

\begin{theorem}\label{thm:discrete-full}
Let $\mathcal{C}$ denote the class of discrete distributions with support size at most $k$. For sufficiently large $k,$ and $m = n+O\left(\frac{n}{\sqrt{k}}\right)$, $\mathcal{C}$ admits an $\left(n, m\right)$ amplification procedure.

This bound is tight up to constants, i.e., there is a constant $c$, such that for every sufficiently large $k$,  $\mathcal{C}$ does not admit an $\left(n, n+\frac{cn}{\sqrt{k}}\right)$amplification procedure.
\end{theorem}

Our amplification procedure for discrete distributions is extremely simple: roughly, we generate additional  samples from the empirical distribution of the initial set of $n$ samples, and then randomly shuffle together the original and the new samples. For technical reasons, we do not exactly sample from the empirical distribution but from a suitable modification which facilitates the analysis. 

Our second result concerns $d$-dimensional Gaussian distributions with unknown mean and fixed covariance. We show that we can amplify even with only $O(\sqrt{d})$ samples from the distribution. In contrast, learning to small constant total variation distance requires $\Theta(d)$ samples. Unlike the discrete setting, however, we do not get optimal amplification in this setting by generating additional samples from the empirical distribution of the initial set of $n$ samples, and then randomly shuffling together the original and new samples. Moreover, we show a lower bound proving that, for  $n=o(d/\log d)$ there is no $(n,n+1)$ amplification procedure which always returns a superset of the original $n$ samples.  Curiously, however, the procedure that generates new samples from the empirical distribution, and then randomly shuffles together the new and old samples, is able to amplify at $n= \Omega(d/\log d)$, even though learning is not possible until $n=\Theta(d)$.  Additionally, as $n$ goes from $10 \frac{d}{\log d}$ to $1000 \frac{d}{\log d}$, this amplification procedure goes from being unable to amplify at all, to being able to amplify by nearly $\sqrt{d}$ samples.  This is formalized in the following proposition. %  (when learning is still not possible), though the amplification achieved is slightly sub-optimal as compared to that in Theorem \ref{thm:gaussian_full}.

\begin{proposition}\label{prop:gaussian_modify_full}
Let $\mathcal{C}$ denote the class of $d-$dimensional Gaussian distributions with unknown mean $\mu$ and  covariance $\Sigma$. There is an absolute constant, $c$, such that for sufficiently large $d$, if $n \le \frac{cd}{\log d},$ there is no $(n,n+1)$ amplification procedure that always returns a superset of the original $n$ points.  % large $d$, $\mathcal{C}$ d $n< O(d/\log d)$, $\mathcal{C}$ does not admit an $(n,n+1)$ amplification procedure which does not modify the input samples.

On the other hand, there is a constant $c'$ such that for any $\eps$, for $n = \frac{d}{\eps \log d}$, and for sufficiently large $d$, there is an $\left(n,n+c'n^{\frac{1}{2}-9\eps}\right)$ amplification protocol for $\mathcal{C}$ that returns a superset of the original $n$ samples.
\end{proposition}
The above proposition suggests that to be able to amplify at input size $n = o(d/\log d)$, one must modify the input samples. A naive way to modify the input samples is to discard all the original $n$ samples and generate $m$ new samples from the distribution $N(\hat{\mu},\Sigma)$, where $\hat{\mu}$ is empirical mean $\hat{\mu}$ of the original set $X_n$. However this does not even give an $(n, n)$ amplification procedure for any value of $n$. To achieve optimal amplification in the Gaussian case, the amplifier first computes the empirical mean $\hat{\mu}$ of the original set $X_n$, and then draws $m-n$ new samples from $N(\hat{\mu},\Sigma)$. We then shift the original $n$ samples to ``decorrelate'' the original set and the new samples; intuitively, this step hides the fact that the $m-n$ new samples were generated based on the empirical mean of the original samples.  The final set of returned samples consists of the shifted versions of the $n$ original samples along with the $m-n$ freshly generated ones.  This procedure gives $(n, n+O(\frac{n}{\sqrt{d}}))$ amplification, and we also show that this amplification is tight up to constant factors.

\begin{theorem}\label{thm:gaussian_full} % Brian is being lazy and not defining a macro and just copy and pasting this. If you change it here, change it in 2 other places!
Let $\mathcal{C}$ denote the class of $d-$dimensional Gaussian distributions $N\left(\mu, \Sigma\right)$ with unknown mean $\mu$ and fixed covariance $\Sigma$. For all $d,n>0$ and $m = n+O\left(\frac{n}{\sqrt{d}}\right)$, $\mathcal{C}$ admits an $\left(n, m\right)$ amplification procedure. 

This bound is tight up to constants, i.e., there is a fixed constant $c$ such that for all $d,n>0$, $\mathcal{C}$ does not admit an $\left(n, m\right)$ amplification procedure for $m\ge n+\frac{cn}{\sqrt{d}}$.
\end{theorem}

\vspace{-10pt}
\subsection{Open Directions}\label{sec:open}

From a technical perspective, there are a number of natural open directions for future work, including establishing tight bounds on amplification for other natural distribution classes, such as $d$ dimensional Gaussians with unknown mean and covariance.  More conceptually, it seems worth getting a broader understanding of the range of potential amplification algorithms, and the settings to which each can be applied.   

\vspace{-5pt}
\paragraph{Weaker or More Powerful Verifiers?}
Our results showing that non-trivial amplification is possible even in the regime in which learning is not possible, rely on the modeling assumption that the verifier gets no information about the amplifier's training set, $X_n$ (the set of $n$ i.i.d. samples).  If this dataset is revealed to the verifier, then the question of amplification is equivalent to learning.  This prompts the question about a middle ground, where the verifier has some information about the set $X_n$, but does not see the entire set; this middle ground also seems the most practically relevant (e.g. how much do I need to know about a GAN's training set to decide whether it actually understands a distribution of images?).

{\centering
\begin{quote}\emph{How does the power of the amplifier vary depending on how much information the verifier has about $X_n$?  If the verifier is given a uniformly random subsample of $X_n$ of size $n' \ll n,$ how does the amount of possible amplification scale with $n'$?}\end{quote}}

Rather than considering how to increase the power of the verifier, as the above question asks, it might also be worth considering the consequences of decreasing either the computational power, or information theoretic power of the verifier.  

{\centering \begin{quote} \emph{If the verifier, instead of knowing distribution $D$,  receives only a set of independent draws from $D$, how much more power does this give the amplifier?  Alternately, if the verifier is constrained to be an efficiently computable function, does this provide additional power to the amplifier in any natural settings?}\end{quote}}

%\begin{comment}
\vspace{-15pt}
\paragraph{Better Amplifiers in the Discrete Setting.}
In the discrete distribution setting, our amplification results are tight (to constant factors) in a worst-case sense, and our amplifier essentially just returns the original $n$ samples, together with additional samples drawn from the empirical distribution of those $n$ samples, and then randomly permutes the order of these datapoints.  This begs the question: \emph{In the case of discrete distributions, is there any benefit to considering more sophisticated amplification schemes?}    Below we sketch one example motivating a more clever amplification approach.

\begin{example}\label{example:good-turing}
Consider obtaining $n$ samples corresponding to independent draws from a discrete distribution that puts probability $p \gg 1/n$ on a single domain element, and with probability $1-p$ draws a sample from the uniform distribution over some infinite discrete domain.  If $p<2/3,$ then the amplification approach that adds samples from the empirical distribution of the data to the original set of samples, will fail.  Indeed, with probability at least $1/3$ it will introduce a second sample of one of the ``rare'' elements, and such samples can be rejected by the verifier.   For this setting, the optimal amplifier would always introduce extra samples corresponding to the element of probability $p$.
\end{example}

The above example motivates a more sophisticated amplification strategy for the discrete distribution setting.  Approaches such as Good-Turing frequency estimation, or more modern variants of it, adjust the empirical probabilities to more accurately reflect the true probabilities (see e.g.~\cite{good1953population,orlitsky2003always,valiant2016instance}).  Indeed, in a setting such as Example~\ref{example:good-turing}, based on the fact that only one domain element is observed more than once, it is easy to conclude that the total probability mass of all the elements observed just once, is likely at most $O(1/n)$, which implies that a successful amplification scheme cannot duplicate any of them.  While inserting samples from a Good-Turing adjusted empirical distribution will not improve the amplification in a worst-case sense for discrete distributions with a bounded support size, such schemes seem \emph{strictly} better than the schemes we currently analyze.  The following question outlines one potential avenue for quantifying this, along the lines of the recent work on ``instance optimal'' distribution testing and estimation (see e.g.~\cite{acharya2012competitive,orlitsky2015competitive,valiant2016instance}): 

{\centering
\begin{quote}\emph{Is there an ``instance optimal'' amplification scheme, which, for every distribution, $D$, amplifies as well as could be hoped?  Specifically, to what extent is there an amplification scheme which performs nearly as well as a hypothetical optimal scheme that knows distribution $D$ up to relabeling/permuting the domain?}\end{quote}}
%\medskip
%\end{comment}
\vspace{-15pt}
\paragraph{Potential Applications of Sample Amplification.}

%reiterate that it doesnt add information, just spreads the information among more datapoints, which might make the information more accessible to certain types of algorithm.  Many commonly used algorithms and heuristics are not information theoretically optimal, despite being very commonly used. For these settings, data amplif..

An interesting future direction is to examine if sample amplification is a useful primitive in settings where the samples are given as input to downstream analysis.   Amplification does not add any new information to the original data, but it could still make the original information more easily accessible to certain types of algorithms which interact with the data in limited ways. For example, many popular algorithms and heuristics are not information theoretically optimal, despite their widespread use. It seems worth examining if  amplification schemes could improve the statistical efficiency of these commonly used methods. Since the amplified samples are ``good'' in an information theoretic sense (they are indistinguishable from true samples), the performance of downstream algorithms \emph{cannot} be significantly hurt.    Below, we provide a toy example of a setting where amplification improves the accuracy of a standard downstream estimator.

\begin{example}\label{example:amphelps}
Given labeled examples, $(x_1,y_1),\ldots,(x_n,y_n)$ drawn from a distribution, $D$, with $x_i \in \mathbb{R}^d$ and $y_i \in \mathbb{R}$, a natural quantity to estimate is the fraction of variance in $y$  explainable as a linear function of $x$:  $$\inf_{\theta \in \mathbb{R}^d} \E_{(x,y)\sim D}[(\theta^T x - y)^2].$$  The standard unbiased estimator for this quantity is the training error of the least-squares linear model, scaled by a factor of $\frac{1}{n-d}.$  This scaling factor makes this estimate unbiased, although the variance is large when $n$ is not much larger than $d$. Figure~\ref{fig:learnability} shows the  expected squared error of this estimator on raw samples, and on $(n,n+2)$ amplified samples, in the case where $x_i \sim N(0,I_d),$ and $y_i = \theta^T x_i + \eta$ for some model $\|\theta\|_2=1$ and independent noise $\eta\sim N(0,\frac{1}{4})$---hence the true value for the ``unexplainable variance'' is  $1/4$. Here, the amplification procedure draws two additional datapoints, $x$ from the isotropic Gaussian with mean equal to the empirical mean, and labels them according to the learned least-squares regression model $\hat{\theta}$ with independent noise of variance $5/n$ times the empirical estimate of the unexplained variance.
\end{example}
%\end{minipage}
%\hspace{.6cm}
%\begin{minipage}{0.4\textwidth}
\begin{wrapfigure}{r}{0.4\textwidth}
\vspace{-\baselineskip}
\centering
%\begin{figure}
%\includegraphics{learnabilityFig.eps}
\includegraphics[width=0.42\textwidth]{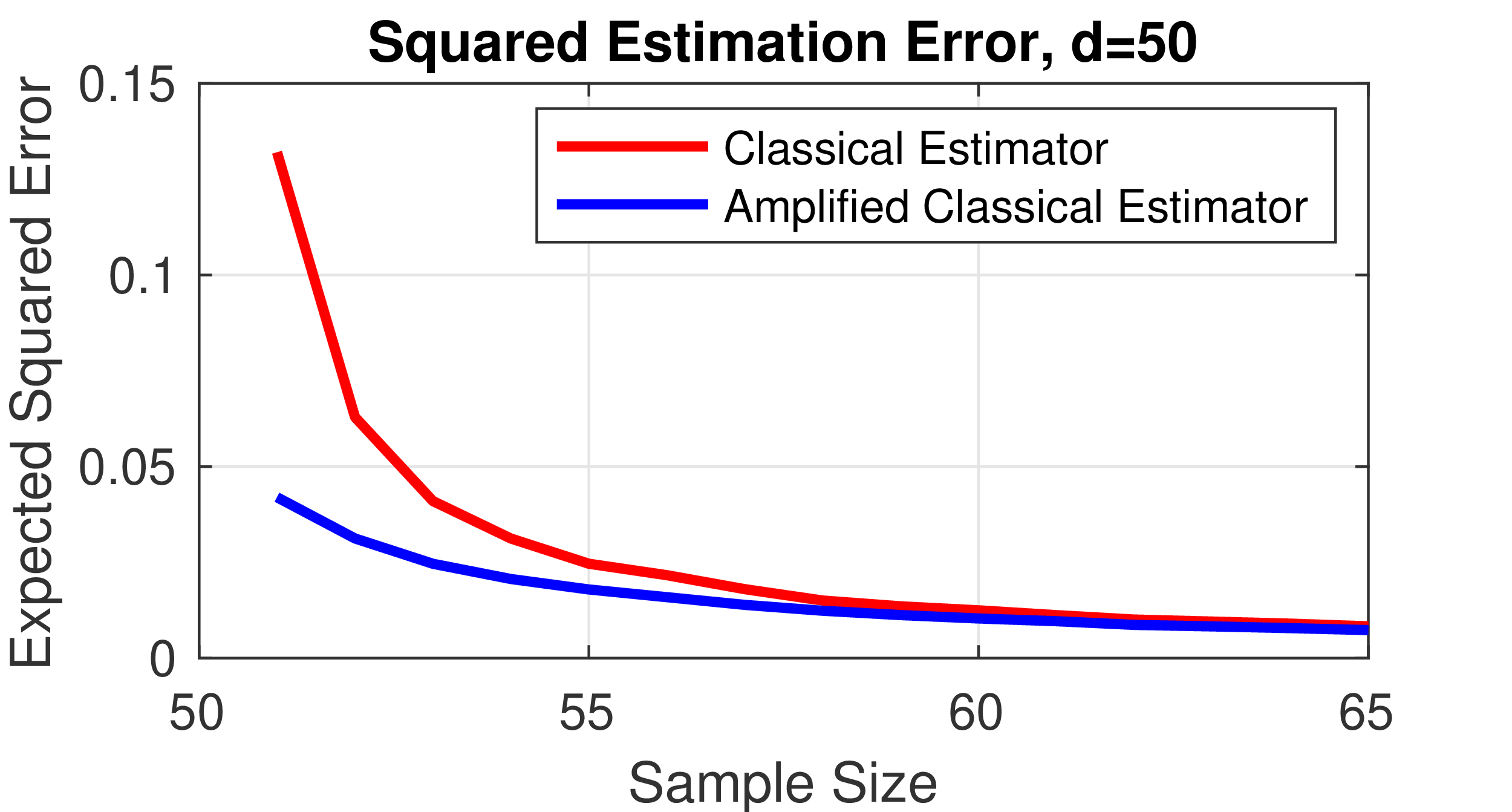}
  \captionof{figure}{\small{Toy example illustrating potential benefit of feeding amplified samples into a commonly used estimator. See Example~\ref{example:amphelps} for a description of the specific setup.  %Here, we are given labeled data pairs $(x,y)$ with $x\sim N(0,I_d)$ and $y= \theta^T x + \eta$ for $\eta \sim N(0,1/4),$ and the goal is to estimate the ``unexplainable variance'', namely $1/4$, from the data.  The red curve is the classic unbiased estimator for this quantity, and the blue curve is the estimate produced by first $(n,n+2)$ amplifying the data, and then applying this classic estimator to the set of $n+2$ amplified samples.
  }}
    \label{fig:learnability}
%\end{figure}
\vspace{-\baselineskip}
\end{wrapfigure}
%\end{minipage}
 %Also, these popular heuristics often do not leverage properties of the underlying data distribution. The amplification scheme could use prior knowledge about the class of distributions where the data lies (such as, say, being a Gaussian) to increase the available data and improve the performance of the downstream heuristic. 

%We aren't amplifying by too much--subconstant factors, but the generated samples are provably "good".  If algorithm consuming data is limited, could attempt to amplify more.  , but if the algorithm consuming the samples is limited

%\medskip

%\hspace{-.6cm}
%\begin{minipage}{0.55\textwidth}

One potential limitation to applications of amplification is that our existing results show that it is only possible to amplify the sample size by sub-constant factors (for the settings considered).  If the algorithm using the amplified data is limited, however, then we could hope for much larger amplification factors. This is reminiscent of the open problem in the previous section on whether larger amplification is possible against weaker classes of verifiers.
%Final paragraph:  In practice, people already doing this sort of thing.  Given our main results that amplification is significantly easier than learning, such pipelines might be more effective than one would initially suspect.  [second point: worth thinking more about how to design these sorts of modular pipelines...]

In practice, there is already growing interest in using generative models for data augmentation to improve classification accuracy \citep{antoniou2017data,frid2018gan,wang2018low,yi2019generative}. Given our results which show that amplification is significantly easier than learning, such pipelines might be more effective than one would initially suspect.  It is also worth thinking more generally about how to design modular data analysis or learning pipelines, where a first component of the pipeline could be an amplifier tailored to the specific data distribution, followed by more generic learning algorithms that do not attempt to leverage structural properties of the data distribution. Such modular pipelines might prove to be significantly easier to develop and maintain, in practice.

\vspace{-5pt}
\paragraph{Implications for Generative Modelling.}

The sample amplification framework has some connections to generative modelling. Generative models such as GANs aim to produce new samples from an unknown distribution $D$ given a training set of size $n$ drawn from $D$. It is tempting to try to relate the amplification setting to GANs by viewing the amplifier and verifier as analogs of the generator and discriminator, respectively.  This is \emph{not} an accurate correspondence:  For GANs, the discriminator typically evaluates examples individually (or in small batches), and often has seen the same training set as the generator, whereas our verifier explicitly evaluates a full set of samples without knowledge of the training samples.  
%evaluates the generated data conditioned on the training data while the verifier in our case does not have access to the amplifier's input data. Therefore, this generator-discriminator framework aims to learn the distribution (in some metric defined by the class of discriminators) rather than do sample amplification. 
%The sample amplification framework might seem closely related to questions in generative modelling. Generative models such as GANs aim to produce new samples from an unknown distribution $D$ given a training set of size $n$ drawn from $D$. While trying to relate the sample amplification setting with the typical generative models such as GANs, one might think that the the amplifier and the verifier corresponds to a GAN's generator and discriminator respectively. However,  note that a discriminator usually evaluates the generated data conditioned on the training data while the verifier in our case does not have access to the amplifier's input data. Therefore, this generator-discriminator framework aims to learn the distribution (in some metric defined by the class of discriminators) rather than do sample amplification. 
The samples generated by a generative model are often evaluated by humans (either manually or algorithmically). This evaluation is usually aimed at understanding the quality of output samples \emph{conditioned on the training data}---if some of the output samples are copies of the training set, this is not satisfactory---which again corresponds to learning rather than sample amplification.

Despite these differences, some ways that generative models are actually used, do closely mirror the amplification setting.  For example, when generative models are used to augment a training set that is used to learn a classifier, both the generated samples and the original dataset are fed into the learning algorithm.  The learning algorithm does not necessarily distinguish between ``new'' and ``old'' samples.  In this setting, it does make sense to evaluate the set of ``new'' and ``old'' samples together, as a single set of ``amplified'' samples,  rather than evaluating the ``new'' samples conditioned on the ``old'' ones.  This exactly corresponds to our amplification formulation.   Given that amplification is often easier than learning, it might be worthwhile trying to develop more techniques that are explicitly trying to amplify, rather than learn.

A second, distinct connection between amplification and GANs, relates to the question of how humans (or algorithms) can evaluate the samples produced by a GAN.  The gap between learning (evaluating the generated samples conditioned on the training set), and amplification (evaluating the generated samples without knowledge of the training set), suggests that in order to truly evaluate the samples produced by a GAN, we would need to closely inspect the training data used by the GAN.  This is obviously impractical in many settings, and motivates some of the questions described above concerning how much access a verifier needs to the input examples in order for there to be a gap between learning, and amplifying.

\subsection{Related Work}

  The question of deciding whether a set of samples consists of independent draws from a specified distribution is one of the fundamental problems at the core of distributional property testing.  Interest in this problem was sparked by the seminal work of Goldreich and Ron~\cite{GR00}, who considered the specific problem of determining whether a set of samples was drawn from a uniform distribution of support size $k$.  This sparked a line of work on the slightly more general problem of ``identity testing'' whether a set of samples was drawn from a specified distribution, $D$, versus a distribution with distance at least $\epsilon$ from $D$~\cite{batu2001testing,paninski2008coincidence,valiant2017automatic,diakonikolas2016new}. Beyond the specific question of identity testing, there is an enormous body of work on other distributional property testing questions, including the ``tolerant'' version of identity testing, as well as the multi-distribution analogs (see e.g.~\cite{batu2013testing,valiant2011testing,chan2014optimal,orlitsky2015competitive,bhattacharya2015testing,levi2013testing,diakonikolas2016new}).  In the majority of these works, the assumption is that the given samples consist of independent draws from some fixed distribution, and the common theme in these results is that such tests can typically be accomplished with far less data than would be required to learn the distribution. While the identity testing problem is clearly related to the amplification problem we consider, these appear to be quite distinct problems.  In particular, in the identity testing setting, the main technical challenge is understanding what statistics of a set of i.i.d. samples are capable of distinguishing samples drawn from the prescribed distribution, versus samples drawn from any distribution that is at least $\eps$-far from that distribution.  In contrast, in the amplification setting, the core question is how the amplifier can leverage a set of independent samples from $D$ to generate a larger set of (presumably) non-independent samples that can successfully masquerade as a set of independent samples drawn from $D$; of course, the catch is that the amplifier must do this in the data regime in which it is impossible for it to learn much about $D$.

 \begin{comment} 
  Beyond the specific question of identity testing, there is an enormous body of work on other distributional property testing questions, including the ``tolerant'' version of identity testing where one wishes to distinguish whether samples were drawn independently from a distribution that is close to a specified distribution, $D$, versus far from $D$, as well as the multi-distribution analogs where one obtains two (or more) sets of i.i.d. samples, drawn respectively from unknown distributions $D_1$, $D_2$, and wishes to distinguish the case that the two distributions are identical (or close) versus have significant total variation distance (see e.g.~\cite{batu2013testing,valiant2011testing,chan2014optimal,orlitsky2015competitive,bhattacharya2015testing,levi2013testing,diakonikolas2016new}).  In the majority of these works, the assumption is that given samples consist of independent draws from some fixed distribution, and the common theme in these results is the punchline that such tests can typically be accomplished with far less data than would be required to learn the distribution in question.
  \end{comment}
  
  Within this line of work on distributional property testing and estimation, there is also a recent thread of work on designing estimators for specific properties (such as entropy, or distance to uniformity), whose performance given $n$ independent draws from the distribution in question is comparable to the expected performance of a naive ``plugin'' estimator (which returns the property value of the empirical distribution) based on $m > n$ independent draws~\cite{valiant2016instance,yi2018data}.  The term ``data amplification'' has been applied to this line of work, although it is a different problem from the one we consider.  In particular, we are considering the extent to which the samples can be used to create a larger set of samples; the work on property estimation is asking to what extent one can craft superior estimators whose performance is comparable to the performance that a more basic estimator would achieve with a larger sample size.

The recent work on \emph{sampling correctors}~\cite{canonne2018sampling} also considers the question of how to produce a ``good'' set of draws from a given distribution.  That work assumes access to independent draws from a distribution, $D$, which is close to having some desired structural property, such as monotonicity or uniformity, and considers how to ``correct'' or ``improve'' those samples to produce a set of samples that appear to have been drawn from a different distribution $D'$ that possesses the desired property (or is closer to possessing the property).  Part of that work also considers the question of whether such a protocol requires access to additional randomness.   

Our formulation of sample amplification as a game between an amplifier and a verifier, closely resembles the setup for \emph{pseudo-randomness}  (see~\cite{vadhan2012pseudorandomness} for a relatively recent survey).   There, the pseudo-random generator takes a set of $n$ independent fair coin flips, and outputs a  longer string of $m > n$ outcomes.  The verifier's job is to distinguish the output of the generator from a set of $m$ independent tosses of the fair coin (i.e. truly random bits).  In contrast to our setting, in pseudo-randomness, both players know that the  distribution in question is the uniform distribution, the catch is that the generator does not have access to randomness, and the verifier is computationally bounded.  Beyond the superficial similarity in setup, we are not aware of any deeper connections between our notion of amplification and pseudorandomness.

Finally, it is also worth mentioning the work of Viola on the complexity of sampling from distributions~\cite{viola2012complexity}.  That work also considers the challenge of generating samples from a specified distribution, though the problem is posed as the computational challenge of producing samples from a specified distribution, given access to uniformly random bits.  One of the punchlines of that work is that there are distributions, such as the distribution over pairs $(x,y)$ where $x$ is a uniformly random length-$n$ string, and $y=parity(x),$ where small circuits can sample from the distribution, yet no small circuit can compute $y=parity(x)$ given $x$.  A different way of phrasing that punchline is that there are distributions that are easy to sample from, for which it is much harder to sample from their  conditional distributions (e.g. in the parity case, sampling $(x,y)$ given $x$ is hard).  
%On the applied machine learning side, there are many tasks that correspond to sampling from conditional distribution, for example sampling from the distribution of images, conditioned on the label being ``cat'', or \emph{infilling} (sampling from the distribution of images, conditioned on a portion of the image).   

\iffalse
Lets not discuss GANs.   We don't want this to seem like a paper on GANs...or do we ?
\fi

\vspace{-5pt}
\section{Algorithms and Proof Overview}

In this section, we describe our algorithms for data amplification for discrete and Gaussian distributions. We also give an intuitive overview of the proofs of both the upper and lower bounds.

\vspace{-5pt}
\subsection{Discrete Distributions with Bounded Support}

We begin by providing some intuition for amplification in the discrete distribution setting, by considering the simple case where the distribution in question is a uniform distribution over an unknown support.  We then describe how our more general amplification algorithm extends this intuition.

\vspace{-5pt}
\paragraph{Intuition via the Uniform Distribution.}
Consider the problem of  generating $(n+1)$ samples from a uniform distribution over $k$ unknown elements,  given a set of $n$ samples from the distribution. Suppose $n\ll \sqrt{k}$. Then with high probability, no element appears more than once in a set of $(n+1)$ samples from $\text{Unif}[k]$. Therefore, as the amplifier only knows $n$ elements of the support with $n$ samples, it cannot produce a set of $(n+1)$ samples such that each element only appears once in the set. Hence, no amplification is possible in this regime. Now consider the case when $n=c\sqrt{k}$ for a large constant $c$. By the birthday paradox, we now expect some elements to appear more than once, and the number of elements appearing twice has expectation $\approx \frac{c^2}{2}$ and standard  deviation $\Theta(c)$. In light of this fact, consider an amplification procedure which takes any element that appears only once in the set $X_n$, adds an additional copy of this element to the set $X_n$, and then randomly shuffles these $n+1$ samples to produce the final set $Z_{n+1}$.  It is easy to verify that the distribution of  $Z_{n+1}$ will be close in total variation distance to a set $X_{n+1}$ of $(n+1)$ i.i.d. samples drawn from the original uniform distribution.  Since the standard deviation of the number of elements in $X_{n+1}$ that appear twice is $\Theta(c)$, intuitively, we should be able to amplify by an additional $\Theta(c)$ samples, by taking $\Theta(c)$ elements which appear only once and repeating them, and then randomly permuting these $n+\Theta(c)$ samples. Note that with high probability, most elements only appear once in the set $X_n$, and hence the previous amplifier is almost equivalent to an amplifier which generates new samples by sampling from the empirical distribution of the original $n$ samples, and then randomly shuffles together the original and new samples. Our amplification procedure for general discrete distributions is based on this sample-from-empirical procedure.

% describe the amplification procedure. explain the poissonization scheme and why it helps. explain why we split into two sets.  describe the proof for showing that whp over first set, the TV distance of the second set is small.
\vspace{-5pt}
\paragraph{Algorithm and Upper Bound.}

To facilitate the analysis, our general amplification procedure which applies to any discrete distribution $D$, deviates from the sample-from-empirical-then-shuffle scheme in two ways. The modifications avoid two sources of dependencies in the sample-from-empirical-then-shuffle scheme which complicate the analysis. 

First, we use the standard ``Poissonization'' trick and go from working with the multinomial distribution to the Poisson distribution---making the  element counts independent for all $\le k$ elements. And second, note that the new samples are dependent on the old samples if we generate the new samples from the empirical distribution. To leverage independence, we instead (i) divide the input samples into two sets, (ii) use the first set to estimate the empirical distribution, (iii) generate new samples using this empirical distribution, and (iv) randomly shuffle these new samples with the samples in the second set. More precisely, we simulate two sets $X_{N_1}$ and $X_{N_2}$, of $\poisson(n/4)$ samples from the distribution $D$, using the original set $X_n$ of $n$ samples from $D$. This is straightforward to do, as a $\poisson(n/4)$ random variable is $\le n/2$ with high probability. We then estimate the probabilities of the elements using the first set $X_{N_1}$, and use these estimated probabilities to generate $R \approx m -n$ more samples from a Poisson distribution, which are then randomly shuffled with the samples in $X_{N_2}$ to produce $Z_{N_2+R}$. Then the set of output samples $Z_m$ just consist of the samples in $X_{N_1}$ concatenated with those in $Z_{N_2+R}$. This describes the main steps in the procedure, more technical details can be found in the full description in Algorithm \ref{alg:discrete}. We show that this procedure achieves a $(n,m)$ amplifier for sufficiently large $k$ and $m=n+O\left(\frac{n}{\sqrt{k}}\right)$.

To prove this upper bound, first note that the counts of each element in a shuffled set $Z_m$ are a sufficient statistics for the probability of observing $Z_m$, as the ordering of the elements is uniformly random. Hence we only need to show that the distribution of the counts in the set $Z_m$ is close in total variation distance to the distribution of counts in a set $X_m$ of $m$ elements drawn i.i.d. from $D$. Since the first set $X_{N_1}$ is independent of the second set $X_{N_2}$, the additional samples added to $X_{N_2}$ are independent of the samples originally in $X_{N_2}$, which avoids additional dependencies in the analysis. Using this independence, we show a technical lemma that with high probability over the first set $X_{N_1}$, the KL-divergence between the distribution of the set $Z_{N_2+R}$ and $D^{N_2+R}$ of $N_2+R$ i.i.d. samples from $D$ is small. Then using Pinsker's inequality, it follows that the total variation distance is also small. The final result then follows by a coupling argument, and showing that the Poissonization steps are successful with high probability.

\vspace{-5pt}
\paragraph{Lower Bound.}

We now describe the intuition for showing our lower bound that the class of discrete distributions with support at most $k$ does not admit an $(n,m)$ amplification scheme for $m\ge n+\frac{cn}{\sqrt{k}}$, where $c$ is a fixed constant. For $n\le \frac k 4$, we show this lower bound for the class of uniform distributions $D=\text{Unif}[k]$ on some unknown $k$ elements. In this case, a verifier can distinguish between true samples from $D$ and a set of amplified samples by counting the number of unique samples in the set. Note that as the support of $D$ is unknown, the number of unique samples in the amplified set is at most the number of unique samples in the original set $X_n$, unless the amplifier includes samples that are outside the support of $D$, in which case the verifier will trivially reject this set.  The expected number of unique samples in $n$ and $m$ draws from $D$ differs by
$\frac {c_1 n }{\sqrt{k}} $, for some fixed constant $c_1$. We use a Doob martingale and martingale concentration bounds to show that the number of unique samples in $n$ samples from $D$ concentrates within a $ \frac{c_2  n }{\sqrt k}$ margin of its expectation with high probability, for some fixed constant $c_2 \ll c_1$. This implies that there will be a large gap between the number of unique samples in $n$ and $m$ draws from $D$. The verifier uses this to distinguish between true samples from $D$ and an amplified set, which cannot have sufficiently many unique samples. 

 Finally, we show that for 
 %\replaced
 { $n> \frac k 4$, a $\Big(n,n+\frac{c'n}{\sqrt{k}}\Big) $ }
 %{ $n> \frac{k}{4}$, a $\Big(n,n+\frac{c'k}{\sqrt{k}}\Big) $ }
 amplification procedure for discrete distributions on $k$ elements implies a 
 %\replaced
 {$(n',n'+\frac{c'n'}{\sqrt{k}}) $ }
 %{$(\frac{k}{ 4},\frac{k}{4}+c'\sqrt{k}) $} 
 amplification procedure for the uniform distribution on $(k-1)$ elements where $n' \approx \frac{k}{8}$. For sufficiently large $c'$ this is a contradiction to the previous part. This reduction follows by considering the distribution which has 
 %\replaced
 {$1 - \frac{k}{8n} $ }
 %{$1 - \frac{k}{4n}$}
 mass on one element and 
 %\replaced
 {$\frac{k}{8n} $ }
 %{$\frac{k}{4n}$}
 mass uniformly distributed on the remaining $(k-1)$ elements. With sufficiently large probability, the number of samples in the uniform section will be  %\replaced
 {$\approx \frac{k}{8}< \frac{k}{4} $ }
 %{$\approx \frac{k}{4}$}, 
 and hence we can apply the previous result.

\subsection{Gaussian Distributions with Unknown Mean and Fixed Covariance}

Given the success of the simple sampling-from-empirical scheme for the discrete case, it is natural to consider the analogous algorithm for $d$-dimensional Gaussian distributions with unknown mean and fixed covariance. In this section, we first show that this analogous procedure achieves non-trivial amplification for $n= \Omega(d/\log d)$. We then describe the idea behind the lower bound that any procedure which does not modify the input samples does not work for $n= o(d/\log d)$. Inspired by the insights from this lower bound, we then discuss a more sophisticated procedure, which is optimal and achieves non-trivial amplification for $n$ as small as $\Omega(\sqrt{d})$.

\vspace{-5pt}
\paragraph{Upper Bound for Algorithm which Samples from the Empirical Distribution.}

%We first note that it is sufficient to consider distributions with identity covariance, because  
Let $\hat{\mu}$ be the empirical mean of the original set $X_n$. Consider the $(n,m)$ amplification scheme which draws $(m-n)$ new samples from $N(\hat{\mu},\Sigma)$ and then randomly shuffles together the original samples and the new samples. We show that for any $\eps$, this procedure---with a small
modification to facilitate the analysis---achieves
$\left(n,n+O \left(n^{\frac{1}{2}-9\eps}\right) \right)$ amplification for $n= \frac{d}{\eps \log d}$ . This is despite the empirical distribution $N(\hat{\mu},\Sigma)$ being $1-o(1)$ far in total variation distance from the true distribution $N(\mu, \Sigma)$, for $n=o(d)$.  

We now provide the proof intuition for this result. First, note that it is sufficient to prove the result for $\Sigma=I$. This is because all the operations performed by our amplification procedure are invariant under linear transformations. The intuition for the result in the identity covariance case is as follows.  Consider $n=\Theta(d/\log d)$. In this case, with high probability, the empirical mean $\hat{\mu}$ satisfies $\norm{\mu-\hat{\mu}} = O(\sqrt{\log d})\le \sqrt{c\log n}$ for a fixed constant $c$. If we center and rotate the coordinate system, such that $\hat{\mu}$ has the coordinates $(\norm{\mu-\hat{\mu}},0,\dots,0)$, then the distribution of samples from $N(\hat{\mu},I)$ and $N(\mu, I)$ only differs along the first axis, and is independent across different axes. Hence, with some technical work, our problem reduces to the following univariate problem: what is the total variation distance between $(n+1)$ samples from the univariate distributions $N(0,1)$ and $\tilde{D}$, where $\tilde{D}$ is a mixture distribution where each sample is drawn from $N(0,1)$ with probability $1-\frac{1}{n+1}$ and from $N(\sqrt{c\log n},1)$ with probability $\frac{1}{n+1}$? We show that the total variation distance between these distributions is small, by bounding the squared Hellinger distance between them. Intuitively, the reason for the total variation distance being small is that, even though one sample from $N(\sqrt{c\log n},1)$ is easy to distinguish from one sample from $N(0,1)$,  for sufficiently small $c$ it is difficult to distinguish between these two samples in the presence of $n$ other samples from $N(0,1)$. This is because for $n$ draws from $N(0,1)$,  with high probability there are $O(n^{1 - c})$ samples in a constant length interval around $\sqrt{c\log n}$, and hence it is difficult to detect the presence or absence of one extra sample in this interval.

\vspace{-5pt}
\paragraph{Lower Bound for any Procedure which Returns a Superset of the Input Samples.}\label{sec:overview_lower_empirical}

We show that procedures which return a superset of the input samples are inherently limited in this Gaussian setting, in the sense that they cannot achieve $(n,n+1)$ amplification for $n\le \frac{cd}{\log d}$, where $c$ is a fixed constant. 

% show why sampling from empirical fails
% To explain the idea behind the lower bound, first consider the above algorithm which samples from the empirical distribution and then mixes the samples together. 

The idea behind the lower bound is as follows. If we consider any arbitrary direction and project a true sample from $N(\mu,I)$ along that direction, then with high probability, the projection lies close to the projection of the mean. However, for input set $X_n$ with mean $\hat{\mu}$, the projection of an extra sample added by any amplification procedure along the direction $\mu-\hat{\mu}$ will be far from the projection of the mean $\mu$. This is because after seeing just $ \frac{cd}{\log d}$ samples, any amplification procedure will have high uncertainty about the location of $\mu$ relative to $\hat{\mu}$.   Based on this, we construct a verifier which can distinguish between a set of true samples and a set of amplified samples, for $n\le \frac{cd}{\log d}$.

More formally, Let $x_i'$ be the $i$-th sample returned by the procedure, and let $\hat{\mu}_{-i}$ be the mean of all except the $i$-th sample. Let ``$\new$'' be the index of the additional point added by the amplifier to the original set $X_n$, hence the amplifier returns the set $\{x_{\new}',X_n\}$. Note that $\hat{\mu}\leftarrow N(\mu, \frac{I}{n})$, hence $\norm{\mu-\hat{\mu}}^2 \approx \frac{d}{n}$ with high probability. Suppose the verifier evaluates the following inner product for the additional point $x_{\new}'$,
\vspace{-3pt}
\begin{align}
    \ip{x_{\new}'-\hat{\mu}_{-\new}}{\mu-\hat{\mu}_{-\new}}. \label{eq:ip_test}
\end{align}
Note that $\hat{\mu}_{-\new}=\hat{\mu}$ as the amplifier has not modified any of the original samples in $X_n$. For a point $x_{\new}'$ drawn from $N(\mu,I)$, this inner product concentrates around $\norm{\mu-\hat{\mu}}^2\approx \frac{d}{n}$.  We now argue that if the true mean $\mu$ is drawn from the distribution $N(0,\sqrt{d}I)$, then the above inner product is much smaller than $\frac{d}{n}$ with high probability over $\mu$. The reason for this is as follows. After seeing the samples in $X_n$, the amplification algorithm knows that $\mu$ lies in a ball of radius $\approx \sqrt{\frac{d}{n}}$  centered at $\hat{\mu}$, but $\mu$ could lie along any direction in that ball. Formally, we can show that if $\mu$ is drawn from the distribution $N(0,\sqrt{d}I)$, then the posterior distribution of $\mu \mid X_n$ is a Gaussian $N(\bar{\mu},\bar{\sigma}I)$ with $\bar{\mu}\approx \hat{\mu}$ and $\bar{\sigma}\approx \frac{1}{n}$. As $\mu-\hat{\mu}$ is a random direction, for any $x_{\new}'$ that the algorithm returns, the inner product in \eqref{eq:ip_test} is $ \approx \norm{x_{\new}'-\hat{\mu}}\norm{\mu-\hat{\mu}}\left(\frac{1}{\sqrt{d}}\right)$ with high probability over the randomness in $\mu \mid X_n$. The verifier checks and ensures that $\norm{x_{\new}'-\hat{\mu}_{-\new}}=\norm{x_{\new}'-\hat{\mu}}\approx \sqrt{d}$. Hence for any $(n,n+1)$ amplification scheme, the inner product in \eqref{eq:ip_test} is at most $\approx \sqrt{\frac{d}{n}}$ with high probability over $\mu \mid X_n$. In contrast, we argued before that this inner product is $\approx \frac{d}{n}$ for a true sample from $N(\mu,I)$.

Finally, note that the algorithm can randomly shuffle the samples, and hence the verifier does the above inner product test for every returned sample $x_i'$, for a total of $(n+1)$ tests. If $(n+1)$ tests are performed, then the inner product is expected to deviate by $\sqrt{\frac{{d \log n}}{n}}$ around its expected value of $\frac{d}{n}$, even for $(n+1)$ true samples drawn for the distribution. But if $n\ll \frac{d}{\log d}$, then $\sqrt{\frac{d}{n}} \ll \frac{d}{n}- \sqrt{\frac{{d \log n}}{n}}$, and hence any $(n,n+1)$ amplification scheme in this regime fails at least one of the following tests with high probability over $\mu$:
\begin{enumerate}
    \item $\forall\; i \in [n+1],  \ip{x_{i}'-\hat{\mu}_{-i}}{\mu-\hat{\mu}_{-i}} \ge \frac{d}{n}- \sqrt{\frac{{d \log n}}{n}}$,
    \item $\forall\; i \in [n+1], \norm{x_{i}'-\hat{\mu}_{-i}}\approx \sqrt{d}$. 
\end{enumerate}
As true samples pass all the tests with high probability, this shows that $(n,n+1)$ amplification without modifying the provided samples is impossible for $n\ll \frac{d}{\log d}$.

\vspace{-5pt}
\paragraph{Optimal Amplification Procedure for Gaussians: Algorithm and Upper Bound.}

The above lower bound shows that it is necessary to modify the input samples $X_n$ to achieve amplification for $n=o(d/\log d)$. {What would be the most naive amplification scheme which does not output a superset of the input samples? One candidate could be an amplifier which first estimates the sample mean $\hat{\mu}$ of $X_n$, and then just outputs $m$ samples from $N(\hat{\mu}, I)$. It is not hard to see that this scheme does not even give a valid $(n, n)$ amplification procedure. The verifier in this case could check the distance between the true mean and the mean of the returned samples, which would be significantly more than expected, with high probability.

How should one modify the input samples then?}  The above lower bound also shows what such an amplification procedure must achieve---the inner product in \eqref{eq:ip_test} should be driven towards its expected value of $\frac{d}{n}$ for a true sample drawn from the distribution. Note that the inner product is too small for the algorithm which samples from the empirical distribution $N(\hat{\mu},I)$ as the generated point $x_{\new}'$ is too correlated with the mean $\hat{\mu}_{-\new}=\hat{\mu}$ of the remaining points. We can fix this by shifting the original points in $X_n$ themselves, to hide the correlation between $x_{\new}'$ and the original mean $\hat{\mu}$ of $X_n$. The full procedure is quite simple to state, and is described in Algorithm \ref{alg:gaussian1}. Note that unlike our  other amplification procedures, this  procedure does not involve any random shuffling of the samples. We show that this procedure achieves $(n,m)$ amplification for all $d>0$ and $m=n + O\left(\frac{n}{\sqrt{d}}\right)$. 
\vspace{-3pt}
\alggaussian

We now provide a brief proof sketch for this upper bound, for the case when $m=n+1$. First, we show that the returned samples in $Z_m$ can also be thought of as a single sample  from a  $(m \times d)$-dimensional Gaussian distribution $N\Big(\underbrace{\left(\mu, \mu, \dots, \mu\right)}_{m \text{ times}}, \tilde{\Sigma}_{md \times md}\Big)$, as the returned samples are linear combinations of Gaussian random variables. Hence, it is sufficient to find their mean and covariance, and use that to bound their total variation distance to true samples from the distribution (which can also be though  of as a single sample from a $(d \times m)$-dimensional Gaussian distribution $N\Big(\underbrace{\left(\mu, \mu, \dots, \mu\right)}_{m \text{ times}}, I_{md \times md}\Big)$). Therefore, our problem reduces to ensuring that the total variation distance between these two Gaussian distributions is small, and this distance is proportional to $\norm{\tilde{\Sigma}_{md \times md} - I_{md \times md}}_{\text{F}}$. Our modification procedure removes the correlations between the original samples and the generated samples to ensure that the non-diagonal entries of $\tilde{\Sigma}_{md \times md}$ are small, and hence the total variation distance is also small. For example, the original correlation between the first coordinates of the original sample $x_1$ and the generated sample $x_{n+1}'$ is too large, but it is easy to verify that the correlation between the first coordinates of the modified sample $x_{1}'=x_1-\frac{x_{n+1}'-\hat{\mu}}{n}$ and the generated sample $x_{n+1}'$ is zero.

\vspace{-5pt}
\paragraph{General Lower Bound for Gaussians.}

We show a lower bound that there is no $(n,m)$ amplification procedure for Gaussian distibutions with unknown mean for $m\ge n+\frac{cn}{\sqrt{d}}$, where $c$ is a fixed constant. The intuition behind the lower bound is that any such amplification procedure could be used to find the true mean $\mu$ with much smaller error than what is  possible with $n$ samples.  %The idea behind the lower bound is similar to the lower bound for procedures which return a superset of the input samples in Section \ref{sec:overview_lower_empirical}. 

To show this formally, we define a verifier such that for $\mu\leftarrow N(0,\sqrt{d}I)$ and $m>n+\frac{cn}{\sqrt{d}}$, $m$ true samples from $N(\mu,I)$ are accepted by the verifier with high probability over the randomness in the samples, but $m$ samples generated by any $(n,m)$ amplification scheme are rejected by the verifier with high probability over the randomness in the samples and $\mu$. %As before, let $x_i'$ be the $i$-th sample returned by the procedure, and let $\hat{\mu}_{-i}$ be the mean of all except the $i$-th sample. 
In this case, the verifier only needs to evaluate the squared distance $\norm{\mu- \hat{\mu}_{m}}^2$ of the empirical mean $\hat{\mu}_m$ of the returned samples from the true mean $\mu$, and accept the samples if and only if this squared distance is less than $\frac{d}{m} + \frac{c_1\sqrt{d}}{{m}}$ for some fixed constant $c_1$.  It is not difficult to see why this test is sufficient. Note that for $m$ true samples drawn from $N(\mu,I)$, $\norm{\mu- \hat{\mu}_{m}}^2 = \frac{d}{m}\pm O\left(\frac{\sqrt{d}}{m}\right)$. Also, the squared distance $\norm{\mu- \hat{\mu}^2}$ of the mean $\hat{\mu}$ of the original set $X_n$ from the true mean $\mu$ is concentrated around $\frac{d}{n}\pm O\left(\frac{\sqrt{d}}{n}\right)$. Using this, for $m>n+\frac{cn}{\sqrt{d}}$, we can show that no algorithm can find a $\hat{\mu}_{m}$ which satisfies $\norm{\mu- \hat{\mu}_{m}}^2 \le \frac{d}{m}\pm O\left(\frac{\sqrt{d}}{m}\right)$ with decent probability over $\mu\leftarrow N(0,\sqrt{d}I)$. This is because the algorithm only knows $\mu$ up to squared error $\frac{d}{n}\pm O\left(\frac{\sqrt{d}}{n}\right)$ based on the original set $X_n$. % This is because $\frac{d}{m}\pm \frac{\sqrt{d}}{m}\ll \frac{d}{n}\pm \frac{\sqrt{d}}{n}$ for $m>n+\frac{cn}{\sqrt{d}}$. 

\section{Proofs: Gaussian with Unknown Mean and Fixed Covariance}
\subsection{Upper Bound}
In this section, we prove the upper bound in Theorem \ref{thm:gaussian_full} by showing that Algorithm \ref{alg:gaussian1} can be used as a $(n, n + \frac{n}{\sqrt d})$ amplification procedure.

First, note that it is sufficient to prove the theorem for the case when input samples come from an identity covariance Gaussian. This is because, for the purpose of analysis we can transform our samples to those coming from indentity covariance Gaussian, as our amplification procedure is invariant to linear transformations to samples. In particular, let $f_\Sigma$ denote our amplification procedure for samples coming from $N(\mu, \Sigma)$, and, $Y_n = (y_1, y_2, \dots, y_n)$ denote the random variable corresponding to $n$ samples from $N(\mu, \Sigma)$. Let $X_n = (x_1, x_2, \dots, x_n)$ denote $n$ samples from $N(\mu, I)$, such that $Y_n = \Sigma^\frac{1}{2}(X_n - \mu) + \mu = (\Sigma^\frac{1}{2}(x_1 - \mu) + \mu, \Sigma^\frac{1}{2}(x_2 - \mu) + \mu, \dots, \Sigma^\frac{1}{2}(x_n - \mu) + \mu)$. Due to invariance of our amplification procedure to linear transformations, we get that $\Sigma^{\frac{1}{2}}(f_I(X_n) - \mu) + \mu$ is equal in distribution to $ f_{\Sigma}(\Sigma^\frac{1}{2}(X_n - \mu) + \mu) = f_{\Sigma}(Y_n)$. This gives us 
\begin{equation*}
    \begin{split}
        D_{TV}(f_\Sigma(Y_n), Y_m) &=  
        D_{TV}(f_\Sigma(\Sigma^\frac{1}{2}(X_n - \mu) + \mu), \Sigma^\frac{1}{2}(X_m - \mu) + \mu)\\
        &= D_{TV}(\Sigma^{\frac{1}{2}}(f_I(X_n) - \mu) + \mu, \Sigma^\frac{1}{2}(X_m - \mu) + \mu)\\
        &\leq D_{TV}(f_I(X_n), X_m),
    \end{split}
\end{equation*}
where the last inequality is true because the total variation distance between two distributions can't increase if we apply the same transformation to both the distributions. Hence, we can conclude that it is sufficient to prove our results for identity covariance case. This is true for both the amplification procedures for Gaussians that we have discussed. So in this whole section, we will work with identity covariance Gaussian distributions.

\begin{proposition}\label{prop:gaussian-ub-main}
Let $\mathcal{C}$ denote the class of $d-$dimensional Gaussian distributions $N\left(\mu, I\right)$ with unknown mean $\mu$. For all $d,n>0$ and $m = n+O\left(\frac{n}{\sqrt{d}}\right)$, $\mathcal{C}$ admits an $\left(n, m\right)$ amplification procedure.
\end{proposition}

%\begin{theorem}\label{gaussian-ub}
%Let $\mathcal{C}$ denote the class of $d-$dimensional Gaussian distributions with unknown mean $\mu$ and identity covariance, $N\left(\mu, I\right)$. For all $d>0$ and $m = n+O\left(\frac{n}{\sqrt{d}}\right)$, $\mathcal{C}$ admits an $\left(n, m\right)$ amplification procedure.%\VS{Can you check if we actually use that $d$ is large for this ? It seems we don't, and we also don't need it for this lower bound. Also, is there any restriction on $n$?}
%\end{theorem}
\begin{proof}
%We begin by giving a procedure to generate $m$ samples $Z_m = \left(x_1', x_2', \dots, x_m'\right)$, given $n$ i.i.d. samples $X_n = \left(x_1, x_2, \dots, x_n\right)$ drawn from $N\left(\mu, I\right)$. 
The amplification procedure consists of two parts. The first uses the provided samples to learn the empirical mean $\hat \mu$ and generate $m-n$ new samples from $\mathcal N(\hat{\mu}, I)$. The second part adjusts the first $n$ samples to ``hide" the correlations that would otherwise arise from using the empirical mean to generate additional samples. 

Let $\epsilon_{n+1}, \epsilon_{n+2}, \dots, \epsilon_{m}$ be $m-n$ i.i.d. samples generated from $N\left(0, I\right)$, and let $\hat{\mu} = \frac{\sum_{i=1}^n x_i}{n}$. The amplification procedure will return $x_1', \dots, x_m'$ with: 
\begin{equation}
    x_i'=
    \begin{cases}
  x_i - \frac{\sum_{j=n+1}^m \epsilon_j}{n}, & \text{for}\ i \in \{1,2, \dots, n\} \\
      \hat{\mu} + \epsilon_i, & \text{for}\ i \in \{n+1, n+2, \dots, m\}.\\
    \end{cases}
  \end{equation}
We will show later in this proof that subtracting $\frac{\sum_{j=n+1}^m \epsilon_j}{n}$ will serve to decorrelate the first $n$ samples from the remaining samples. 

Let $f_{\mathcal{C},n,m} : S^n  \rightarrow S^m$ be the random function denoting the map from $X_n$ to $Z_m$ as described above, where $S = \R^d$. We need to show 
$$D_{TV}\left(Z_m = f_{\mathcal{C},n,m}\left(X_n\right), X_m\right) \le 1/3,$$
where $X_n$ and $X_m$ denote $n$ and $m$ independent samples from $N\left(\mu, I\right)$ respectively.

For ease of understanding, we first prove this result for the univariate case, and then extend it to the general setting.

So consider the setting where $d=1$. In this case, $X_m$ corresponds to $m$ i.i.d. samples from a Gaussian with mean $\mu$, and variance $1$. $X_m$ can also be thought of as a single sample from an $m-$dimensional Gaussian $N\Big(\underbrace{\left(\mu, \mu, \dots, \mu\right)}_{m \text{ times}}, I_{m \times m}\Big)$. Now, $f_{\mathcal{C},n,m}$  is a map that takes $n$ i.i.d samples from $N\left(\mu,1\right)$, $m-n$ i.i.d samples ($\epsilon_i$)
from $N\left(0,1\right)$, and outputs $m$ samples that are a linear combination of the $m$ input samples. So, $f_{\mathcal{C},n,m}\left(X_n\right)$ can be thought of as a
$m-$dimensional random variable obtained by applying a linear transformation to a sample drawn from
$N\Big(\Big(\underbrace{\mu, \mu, \dots, \mu}_{n \text{ times}}, \underbrace{0, 0,\dots, 0}_{m-n \text{ times}}\Big), I_{m \times m}\Big)$. As a linear transformation applied to Gaussian random variable outputs a Gaussian random variable, we get that $Z_m = \left(x_1', x_2', \dots, x_m'\right)$ is distributed according to $N\left(\tilde{\mu},\Sigma_{m \times m}\right)$, where $\tilde{\mu}$ and $\Sigma_{m \times m}$ denote the mean and covariance. Note that $\tilde{\mu} = \underbrace{\left(\mu, \mu, \dots, \mu\right)}_{m \text{ times}}$ as  

\begin{equation}
    \E[x_i']=
    \begin{cases}
      \E[x_i] - \E \left [\frac{\sum_{j=n+1}^m \epsilon_j}{n} \right ] = \mu - 0 = \mu, & \text{for}\ i \in \{1,2, \dots, n\} \\
      \E[\hat{\mu}] + \E[\epsilon_i] = \mu + 0 = \mu, & \text{for}\ i \in \{n+1, n+2, \dots, m\}.\\
    \end{cases}
\end{equation}
Next, we compute the covariance matrix $\Sigma_{m \times m}$.

For $i=j$, and $i \in \{1,2, \dots, n\}$, we get
\begin{equation*}
\begin{split}
\Sigma_{ii} & = \E[\left(x_i' - \mu\right)^2] \\
 & = \E\left [\left(x_i - \mu\right)^2 \right ] + \E \Bigg [\left(\frac{\sum_{j=n+1}^m \epsilon_j}{n}\right)^2\Bigg]\\
 & = 1 + \frac{m-n}{n^2}.
\end{split}
\end{equation*}

For $i=j$, and $i \in \{n+1,n+2, \dots, n+m\}$, we get
\begin{equation*}
\begin{split}
\Sigma_{ii} & = \E\left [\left(x_i' - \mu\right)^2 \right ] \\
 & = \E \left [\left(\hat{\mu} - \mu\right)^2 \right ] + \E \left [ \epsilon_i^2 \right]\\
 & = \frac{1}{n} +1 .
\end{split}
\end{equation*}

For $i \in \{1,2, \dots, n\}, j \in \{n+1,n+2, \dots, n+m\}$, we get
\begin{equation*}
\begin{split}
\Sigma_{ij} & = \E \left [\left(x_i' - \mu\right)\left(x_j' - \mu\right) \right] \\
 & = \E\left[\left(x_i - \frac{\sum_{k=n+1}^m \epsilon_k}{n} - \mu\right)\left(\hat{\mu} + \epsilon_j - \mu\right)\right]\\
 & = \E[\left(x_i - \mu\right)\left(\hat{\mu} - \mu\right)] - \E\left[\left(\frac{\sum_{k=n+1}^m \epsilon_k}{n}\right)\left(\epsilon_j\right)\right]\\
 & = \frac{1}{n} - \frac{1}{n}\\
 & = 0.
\end{split}
\end{equation*}

For $i, j \in \{1,2, \dots, n\}, i \neq j$, we get
\begin{equation*}
\begin{split}
\Sigma_{ij} & = \E\Big[\left(x_i' - \mu\right)\left(x_j' - \mu\right)\Big] \\
 & = \E\Bigg[\left(x_i - \frac{\sum_{k=n+1}^m \epsilon_k}{n} - \mu\right)\left(x_j - \frac{\sum_{k=n+1}^m \epsilon_k}{n} - \mu\right)\Bigg]\\
 & = \E \left [ (x_i - \mu) (x_j - \mu ) \right ] + \E\left[\left(\frac{\sum_{k=n+1}^m \epsilon_k}{n}\right)^2\right]\\
 & = \frac{m-n}{n^2}.
\end{split}
\end{equation*}

For $i, j \in \{n+1,n+2, \dots, m\}, i \neq j$, we get
\begin{equation*}
\begin{split}
\Sigma_{ij} & = \E[\left(x_i' - \mu\right)\left(x_j' - \mu\right)] \\
 & = \E[\left(\hat{\mu} + \epsilon_i - \mu\right)\left(\hat{\mu} + \epsilon_j - \mu\right)]\\
 & = \E \left [\left(\hat{\mu} - \mu\right)^2 \right ]\\
 & = \frac{1}{n}.
\end{split}
\end{equation*}
This gives us
\[
\Sigma_{m \times m} = \begin{bmatrix} 

    1+\frac{m-n}{n^2} & \frac{m-n}{n^2} & \cdots  &  \frac{m-n}{n^2} & 0 & 0  & \cdots  & 0 \\
    \frac{m-n}{n^2} & 1+\frac{m-n}{n^2} & \cdots & 
    \frac{m-n}{n^2}  & 0 & 0   & \cdots  & 0 \\
    \vdots & \cdots & \cdots & \vdots & \vdots & \cdots & \cdots & \vdots\\
    \vdots & \cdots   & \cdots  & \frac{m-n}{n^2} & \vdots & \cdots  & \cdots & \vdots\\
    \frac{m-n}{n^2} & \cdots  &  \frac{m-n}{n^2} &  1+\frac{m-n}{n^2} & 0 & 0  & \cdots & 0  \\
    0 & \cdots  & \cdots & 0 & 1 + \frac{1}{n} & \frac{1}{n}   & \cdots & \frac{1}{n}\\
    0 & \cdots  & \cdots & 0 &  \frac{1}{n} & 1+\frac{1}{n}   & \cdots & \frac{1}{n}\\
    \vdots & \cdots & \cdots & \vdots & \vdots & \cdots & \cdots & \vdots\\
    \vdots & \cdots & \cdots & \vdots & \vdots & \cdots & \cdots &  \frac{1}{n}\\
    0 & \cdots  & \cdots & 0 &  \frac{1}{n} & \cdots &\frac{1}{n}    & 1+\frac{1}{n}\\
    \end{bmatrix}.
\]
Now, finding $D_{TV}\left(Z_m, X_m\right)$ reduces to computing $D_{TV}\left(N\left(\tilde{\mu}, I_{m \times m}\right), N\left(\tilde{\mu}, \Sigma_{m \times m}\right)\right)$. From \cite[Theorem 1.1]{devroye2018total}, we know that $D_{TV}\left(N\left(\tilde{\mu}, I_{m \times m}\right), N\left(\tilde{\mu}, \Sigma_{m \times m}\right)\right) \leq \min\left(1, \frac{3}{2}||\Sigma - I||_F\right)$. This gives us

\begin{equation}\label{unidim}
\begin{split}
D_{TV}\left(N\left(\tilde{\mu}, I_{m \times m}\right), N\left(\tilde{\mu}, \Sigma_{m \times m}\right)\right) &\leq \min\left(1, \frac{3}{2}||\Sigma - I||_F\right)\\
&\leq \sqrt{\frac{3}{2}\left(\left(\frac{m - n}{n^2}\right)^2 n^2 + \frac{1}{n^2} \left(m-n\right)^2\right)}\\
& = \frac{\sqrt{3}\left(m-n\right)}{n}.
\end{split}
\end{equation}

Now, for $d > 1$, by a similar argument as above, $X_m$ can be thought of as $d$ independent samples from the following $d$ distributions: $N\Big(\underbrace{\left(\mu_1, \mu_1, \dots, \mu_1\right)}_{m \text{ times}}, I_{m \times m}\Big), \dots,N\Big(\underbrace{\left(\mu_d, \mu_d, \dots, \mu_d\right)}_{m \text{ times}}, I_{m \times m}\Big)$.
Or equivalently, as a single sample from $N\Big(\Big(\underbrace{\mu_1, \mu_1, \dots, \mu_1}_{m \text{ times}},\dots ,\underbrace{\mu_d, \mu_d, \dots, \mu_d}_{m \text{ times}}\big), I_{md \times md}\Big)$. Similarly, $Z_m$ can be thought of as $d$ independent samples from $N\Big(\underbrace{\left(\mu_i, \mu_i, \dots, \mu_i\right)}_{m \text{ times}}, \Sigma_{m \times m}\Big)$, or equivalently, a single sample from $N\Big(\Big(\underbrace{\mu_1, \mu_1, \dots, \mu_1}_{m \text{ times}},\dots ,\underbrace{\mu_d, \mu_d, \dots, \mu_d}_{m \text{ times}}\big), \tilde{\Sigma}_{md \times md}\Big)$ where $\tilde{\Sigma}_{md \times md}$ is a block diagonal matrix with block diagonal entries equal to $\Sigma_{m \times m}$ (denoted as $\Sigma$ in the figure).

\[
\tilde{\Sigma}_{md \times md} = \begin{bmatrix} 

    \Sigma_{} & 0 & \cdots  &  \cdots & 0 \\
    0 & \Sigma_{} & 0 & \cdots  &  \vdots \\
    \vdots & 0 & \ddots  &  0 & \vdots \\
    \vdots & \cdots & 0  &  \ddots & 0\\
   0 & \cdots & \cdots  & 0 &  \Sigma_{}
    \end{bmatrix}.
\]
Similar to \eqref{unidim}, we get

\begin{equation*}
\begin{split}
&D_{TV}\Big(
N\Big(\Big(\underbrace{\mu_1, \mu_1, \dots, \mu_1}_{m \text{ times}},\dots ,\underbrace{\mu_d, \mu_d, \dots, \mu_d}_{m \text{ times}}\big), I_{md \times md}\Big)
, N\Big(\Big(\underbrace{\mu_1, \mu_1, \dots, \mu_1}_{m \text{ times}},\dots ,\underbrace{\mu_d, \mu_d, \dots, \mu_d}_{m \text{ times}}\big), \tilde{\Sigma}_{md \times md}\Big)
\Big)\\
& \hspace{130pt} \leq  \min\left(1, \frac{3}{2}||\tilde{\Sigma} - I||_F\right)\\
& \hspace{130pt} \leq \sqrt{d \left(\frac{3}{2}\left(\left(\frac{m - n}{n^2}\right)^2 n^2 + \frac{1}{n^2} \left(m-n\right)^2\right)\right)}\\
& \hspace{130pt} = \frac{\sqrt{3d}\left(m-n\right)}{n}.
\end{split}
\end{equation*}
If we want the total variation distance to be less than $\delta$, we get $m-n = O\left(\frac{n \delta}{\sqrt{d}}\right)$. Setting $\delta = \frac{1}{3}$, we get $m = n+  O\left(\frac{n}{\sqrt{d}}\right)$, which completes the proof.
\end{proof}
\subsection{Lower Bound}
In this section we prove the lower bound from Theorem \ref{thm:gaussian_full} and show that it is impossible to amplify beyond $O \left ( \frac{ n }{\sqrt d} \right )$ samples. 

\begin{proposition}\label{prop:gaussian_lower}
Let $\mathcal{C}$ denote the class of $d-$dimensional Gaussian distributions $N\left(\mu, I\right)$ with unknown mean $\mu$. There is a fixed constant $c$ such that for all sufficiently large $d,n>0$, $\mathcal{C}$ does not admit an $\left(n, m\right)$ amplification procedure for $m\ge n+\frac{cn}{\sqrt{d}}$.
\end{proposition}

%\begin{theorem}\label{thm:gaussian_lower}
%Let $\mathcal{C}$ denote the class of $d-$dimensional Gaussian distributions with unknown mean $\mu$ and identity covariance, $N\left(\mu, I\right)$. Then for a fixed constant $c$, $\mathcal{C}$ does not admit an $(n,m)$ amplification procedure for $m\ge n+\frac{cn}{\sqrt{d}}$, 
%\end{theorem}
\begin{proof}

%We will show that $\mathcal{C}$ does not admit an $(n,m)$ amplification procedure for $m\ge n + 100n/\sqrt{d}$. 
Note that it is sufficient to prove the theorem for $m=n+cn/\sqrt{d}$ for a fixed constant $c$, as an amplification procedure for $m>n+cn/\sqrt{d}$ implies an amplification procedure for $m=n+cn/\sqrt{d}$ by discarding the residual samples. To prove the theorem for $m=n+cn/\sqrt{d}$, we will define a distribution $D_{\mu}$ over $\mu$ and a verifier $v(Z_{m})$ for the distribution $N(\mu,I)$ which takes as input a set $Z_{m}$ of $m$ samples, such that: (i) for all $\mu$, the verifier $v(Z_{m})$ will accept with probability $1-1/e^2$ when given as input a set $Z_{m}$ of $m$ i.i.d. samples from $N(\mu,I)$, (ii) but will reject any $(n,m)$ amplification procedure for $m=n+cn/\sqrt{d}$ with probability $1-1/e^2$, where the probability is with respect to the randomness in $\mu\leftarrow D_{\mu}$, the set $X_n$ and in any internal randomness of the amplifier. Note that by Definition \ref{def2} of an amplification procedure, this implies that there is no $(n,m)$ amplification procedure for $m=n+cn/\sqrt{d}$. \\

We now define the distribution $D_{\mu}$ and the verifier  $v(Z_{m})$. We choose $D_{\mu}$ to be $N(0,\sqrt{d}I)$. Let $\hat{\mu}_{m}$ be the mean of the samples $Z_{m}$ returned by the amplification procedure. The verifier  $v(Z_{m})$ performs the following test, accepts if $\hat{\mu}_{m}$ passes the test, and rejects otherwise---
\begin{align}
\Big|\norm{\hat{\mu}_{m}-\mu}^2-d/m\Big|\le 10\sqrt{d}/m.\label{eq:simple_test}
\end{align}
We first show that $m$ i.i.d. samples from $N(\mu, I)$ pass the above test with probability $1-1/e^2$. We will use the following concentration bounds for a $\chi^2$ random variable $Z$ with $d$ degrees of freedom \cite{laurent2000adaptive,wainwright2015basic},
\begin{align}
\Pr\Big[Z - d \ge 2\sqrt{dt} + 2t\Big] &\le e^{-t},\; \forall \; t>0,\label{eq:chisquare}\\
\Pr\Big[|Z - d|\ge dt] &\le 2e^{-dt^2/8}, \; \forall \; t\in (0,1).\label{eq:chisquare2}
\end{align}
Note that $\hat{\mu}_{m}\leftarrow N(\mu,\frac{I}{m})$ for $m$ i.i.d. samples from $N(\mu, I)$. Hence by using \eqref{eq:chisquare2} and setting $t=10/\sqrt{d}$, 
\begin{align*}
\Pr\Big[\Big|\norm{\hat{\mu}_{m}-\mu}^2-d/m\Big|>10\sqrt{d}/m\Big] \le 1/e^3.
\end{align*}
Hence $m$ i.i.d. samples from $N(\mu, I)$ pass the  test with probability at least $1-1/e^2$.

We now show that for $\mu$ sampled from $ D_{\mu}=N(0,\sqrt{d}I)$, the verifier rejects any $(n,m)$ amplification procedure for $m=n+ cn/\sqrt{d}$ with high probability over the randomness in $\mu$. Let $D_{\mu \mid X_n}$ be the posterior distribution of $\mu$ conditioned on the set $X_n$. We will show that for any set $X_n$ received by the amplifier, the amplified set $Z_{m}$ is accepted by the verifier with probability at most $1/e^2$ over $\mu\leftarrow D_{\mu \mid X_n}$. This implies that with probability $1-1/e^2$ over the randomness in $\mu\leftarrow D_{\mu}$, the set $X_n$ and any internal randomness in the amplifier, the amplifier cannot output a set $Z_{m}$ which is accepted by the verifier, completing the proof of  Proposition \ref{prop:gaussian_lower}.

To show the above claim, we first find the posterior distribution $D_{\mu \mid X_n}$ of $\mu$ conditioned on the amplifier's set $X_n$. Let $\mu_0$ be the mean of the set $X_n$. By standard Bayesian analysis (see, for instance, \citep{murphy2007conjugate}), the posterior distribution $D_{\mu \mid X_n}=N(\bar{\mu},\bar{\sigma}^2 I)$, where,
\begin{align*}
\bar{\mu} = \frac{n}{n+1/\sqrt{d}}\mu_0, \quad \bar{\sigma}^2 = \frac{1}{n+1/\sqrt{d}}.
\end{align*}
%We now show that for $\mu\leftarrow N(0,\sqrt{d}I)$, the verifier will reject the samples produced by a $(n,m)$ amplification scheme for $m=n+100n/\sqrt{d}$ with probability $1-1/e^2$ over the randomness in $\mu$. This implies that there is no $(n,m)$ amplification scheme for $m=n+100n/\sqrt{d}$. 

%The posterior distribution $D_{\mu \mid X_n}$ of the mean given the samples $X_n$ is the same as in the previous case. 
We show that any set $Z_m$ returned by the amplifier for $m=n+100n/\sqrt{d}$ fails the test \eqref{eq:simple_test} with probability $1-1/e^2$ over the randomness in $\mu \mid X_n$. We expand $\norm{\hat{\mu}_{m}-\mu}^2$ in the test as follows,
\begin{align*}
\norm{\hat{\mu}_{m}-\mu}^2 &= \norm{\hat{\mu}_{m}-\overline{\mu} - (\mu - \overline{\mu})}^2\\
&=\norm{\hat{\mu}_{m}-\bar{\mu}}^2 -2 \ip{\hat{\mu}_{m}-\bar{\mu}}{{\mu}-\bar{\mu}} + \norm{\mu-\bar{\mu}}^2.
\end{align*}
By using \eqref{eq:chisquare2} and setting $t=10/\sqrt{d}$,  with probability $1-1/e^3$,
\begin{align*}
\norm{\mu-\bar{\mu}}^2&\ge \frac{d}{n+1/\sqrt{d}}-\frac{10\sqrt{d}}{n+1/\sqrt{d}}\\
&\ge \Big(\frac{d}{n}\Big) \Big(1-\frac{1}{n\sqrt{d}}\Big) -\frac{10\sqrt{d}}{n}\\
&= d/n -\sqrt{d}/n^2 -10\sqrt{d}/n\\
&\ge d/n  -12\sqrt{d}/n.
\end{align*}
As $\mu \mid X_n \leftarrow N(\bar{\mu},\bar{\sigma}^2)$, $\ip{\hat{\mu}_{m}-\bar{\mu}}{{\mu}-\bar{\mu}}$ is distributed as $N(0,\bar{\sigma}^2\norm{\hat{\mu}_{m}-\bar{\mu}}^2)$. Hence with probability $1-1/e^3$, $\ip{\hat{\mu}_{m}-\bar{\mu}}{{\mu}-\bar{\mu}}\le 10\norm{\hat{\mu}_{m}-\bar{\mu}}/\sqrt{n+1/\sqrt{d}}\le 10\norm{\hat{\mu}_{m}-\bar{\mu}}/\sqrt{n}$. Therefore, with probability $1-2/e^3$, 
\begin{align*}
\norm{\hat{\mu}_{m}-\mu}^2 \ge \norm{\hat{\mu}_{m}-\bar{\mu}}^2 -(20/\sqrt{n}) \norm{\hat{\mu}_{m}-\bar{\mu}}+ d/n-12\sqrt{d}/n.
\end{align*}
We claim that $\norm{\hat{\mu}_{m}-\bar{\mu}}^2-20{\norm{\hat{\mu}_{m}-\bar{\mu}}/\sqrt{n}}\ge-100/n$. To verify, note that $\norm{\hat{\mu}_{m}-\bar{\mu}}^2-20{\norm{\hat{\mu}_{m}-\bar{\mu}}/\sqrt{n}}+100/n$ is a non-negative quadratic function in $\norm{\hat{\mu}_{m}-\bar{\mu}}$. Therefore, with probability at least $1-2/e^3$,
\begin{align*}
\norm{\hat{\mu}_{m}-\mu}^2 \ge -100/n + d/n-\sqrt{d}/n^2-10\sqrt{d}/n \ge d/n-20\sqrt{d}/n.
\end{align*}
To pass \eqref{eq:simple_test}, $\norm{\hat{\mu}_{m}-\mu}^2 \le d/m + 10\sqrt{d}/m$. Therefore, if an amplifier passes the test with probability greater than $1-2/e^3$ over the randomness in $\mu \mid X_n$  for $m= n + 100n/\sqrt{d}$, then,
\begin{align*}
 &d/n-20\sqrt{d}/n \le \norm{\hat{\mu}_{m}-\mu}^2 \le d/m + 10\sqrt{d}/m, \\
 \implies& d/n-20\sqrt{d}/n \le d/m + 10\sqrt{d}/m,\\
 \implies& d/n-20\sqrt{d}/n \le d/(n+100n/\sqrt{d}) + 10\sqrt{d}/(n+100n/\sqrt{d}),\\
 \implies& d/n-20\sqrt{d}/n \le d/n(1+100/\sqrt{d})^{-1} + 10\sqrt{d}/n(1+100/\sqrt{d})^{-1},\\
 \implies& d/n-20\sqrt{d}/n \le d/n(1-50/\sqrt{d}) + 10\sqrt{d}/n(1-50/\sqrt{d}), \\
 \implies& -20\sqrt{d}/n \le -40\sqrt{d}/n -1000/n,\\
 \implies& -20\sqrt{d}/n \le -30\sqrt{d}/n ,
\end{align*}
which is a contradiction. Hence for $m= n + 100n/\sqrt{d}$, every $(n,m)$ amplifier is rejected by the verifier  with probability greater than $1-1/e^2$ over the randomness in $\mu$, the set $X_n$, and any internal randomness of the amplifier.

\end{proof}

\subsection{Upper Bound for Procedures which Returns a Superset of the Input Samples}

In this section we prove the upper bound in Proposition \ref{prop:gaussian_modify_full}. The algorithm itself is presented in Algorithm \ref{alg:gaussian2}. Before we proceed with the proof we prove a brief lemma that will be useful for bounding the total variation distance.

\begin{lemma}\label{lem:coupling}
Let $X, Y_1, Y_2$ be random variables such that with probability at least $1 - \epsilon$ over X, $D_{TV}(Y_1 | X, Y_2| X) \leq \epsilon'$, then $D_{TV}((X, Y_1), (X, Y_2)) \leq \epsilon + \epsilon'$.
\end{lemma}
\begin{proof}
From the definition of total variation distance, we know
\begin{equation*}
\begin{split}
D_{TV}((X, Y_1), (X, Y_2)) &= \frac{1}{2}\sum_{x,y} \left \lvert \Pr((X, Y_1) = (x,y)) - \Pr((X, Y_2) = (x, y))) \right \rvert\\
& = \frac{1}{2}\sum_{x,y} \Pr(X = x) \left \lvert \Pr \left (Y_1 = y \mid X = x \right ) - \Pr(Y_2 = y \mid X = x) \right \rvert\\
& = \sum_{x} \Pr(X = x) \ \frac{1}{2}\sum_{y} \lvert \Pr(Y_1 = y \mid X = x) - Pr(Y_2 = y \mid X = x) \rvert\\
& = \sum_x \Pr(X = x) \ d_{TV}(Y_1 \mid X=x, Y_2 \mid X=x).
\end{split}
\end{equation*}
Since with probability $(1 - \epsilon)$  over $X$, $d_{TV}(Y_1 \mid X, Y_2 \mid X)$ is at most $\epsilon'$, and total variation distance is always bounded by 1, we get $\sum_x Pr(X = x) \ d_{TV}(Y_1 \mid X=x, Y_2 \mid X=x) \leq  (1-\epsilon)\epsilon' + \epsilon \leq \epsilon' + \epsilon$. 

This same proof with summations appropriately replaced with integrals will go through when the random variables in consideration are defined over continuous domains.
\end{proof}

Now we prove the upper bound from Proposition \ref{prop:gaussian_modify_full}. As in Proposition \ref{prop:gaussian-ub-main}, it is sufficient to prove this bound only for the case of identity covariance gaussians as our algorithm in this case is also invariant to linear transformation.
\begin{proposition}
Let $\mathcal{C}$ denote the class of $d-$dimensional Gaussian distributions $N(\mu, I)$ with unknown mean $\mu$. There is a constant $c'$ such that for any $\epsilon$, and $n = \frac{d}{\eps \log d}$, and for sufficiently large $d$, there is an $\left(n,n+c'n^{\frac{1}{2}-9\eps}\right)$ amplification protocol for $\mathcal{C}$ that returns a superset of the original $n$ samples.
\end{proposition}

\alggaussiannonmod

%\begin{theorem}\label{thm:gaussian-ub-weak}
%Let $\mathcal{C}$ denote the class of $d-$dimensional gaussian distributions with unknown mean $\mu$ and identity covariance, $N\left(\mu, I\right)$. For  $n = \frac{2d}{\eps^2 \log d}, m = n+O\left({n^{\frac{1}{2} - 3\eps^2}}\right)$, for all $\eps < 0.1$, and for $d$ large enough, $\mathcal{C}$ admits an $\left(n, m\right)$ amplification procedure that does not modify the input samples.
%\end{theorem}
\begin{proof}
Let $m = n+r$ ,where $r = O\left(n^{\frac{1}{2} - 9\epsilon}\right)$. We begin by describing the procedure to generate $m$ samples $Z_m = \left(x_1', x_2', \dots, x_m'\right)$, given $n$ i.i.d. samples $X_{n} = \left(x_1, x_2, \dots, x_{n}\right)$ drawn from $N\left(\mu, I\right)$. Let $\tilde{\mu} = \sum_{i=1}^{n/2} \frac{x_i}{n/2}$ denote the mean of first $\frac{n}{2}$ samples in $X_{n}$. For distributions $P$ and $Q$, let $(1-\alpha) P + \alpha Q$ denote the mixture distribution where $(1 - \alpha)$ and $\alpha$ are the respective mixture weights. 

We first describe how to generate $Z_m' = (x_1'',x_2'', \dots, x_m'')$, given $n$ i.i.d samples $X_{n}$.
For $i \in \{1, 2, \dots, \frac{n}{2}\}$, we set $x_i''= x_i$. For $i \in \{\frac{n}{2}+1, \frac{n}{2}+2, \dots, m\}$, we set $x_i''$ to a random independent draw from the mixture distribution $\left(1 - \frac{10r}{r+\frac{n}{2}}\right)N(\mu, I_{d \times d}) + \frac{10r}{r+\frac{n}{2}}N(\tilde{\mu}, I_{d \times d})$.

Now, the construction of $Z_m$ is very similar to $Z_m'$ except that we don't have access to $N(\mu, I_{d \times d})$ to sample points from the mixture distribution. So, for $Z_m$, set $x_i'= x_i$ for $i \in \{1, 2, \dots, \frac{n}{2}\}$. For $i \in \{\frac{n}{2}+1, \frac{n}{2}+2, \dots, m\}$, we use samples from $(x_{\frac{n}{2}+1},x_{\frac{n}{2}+2}, \dots, x_{n} )$ instead of producing new samples from $N(\mu, I_{d \times d})$. With probability $\left(1 - \frac{10r}{r+\frac{n}{2}}\right)$, we generate a random sample without replacement from $\left(x_{\frac{n}{2}+1}, x_{\frac{n}{2}+2}, \dots, x_{n}\right)$, and with probability $\frac{10r}{r+\frac{n}{2}}$ we generate a sample from $N(\tilde{\mu},I)$, and set $x_i'$ equal to that sample.  As we sample from $(x_{\frac{n}{2}+1}, x_{\frac{n}{2}+2}, \dots, x_{n})$ without replacement, we can generate only $\frac{n}{2}$ samples  this way. The expected number of samples needed is $(\frac{n}{2}+r)(1 - \frac{10r}{r+\frac{n}{2}}) = \frac{n}{2} - 9r$, and with high probability, we won't need more than $\frac{n}{2}$ samples. If the total number of required samples  from $\left(x_{\frac{n}{2}+1}, x_{\frac{n}{2}+2}, \dots, x_{n}\right)$ turns out to be more than $\frac{n}{2}$, we set $x_i$ to an arbitrary $d-$dimensional vector (say $x_1$) but this happens with low probability, leading to insignificant loss in total variation distance.

Let $X_m$ denote the random variable corresponding to $m$ i.i.d. samples from $N(\mu, I)$. We want to show that $D_{TV}(X_m, Z_m)$ is small. By triangle inequality, $D_{TV}(X_m, Z_m) \leq D_{TV}(X_m, Z_m') + D_{TV}(Z_m', Z_m)$.

We first bound $D_{TV}(Z_m, Z_m')$. Let $Y, Y' \leftarrow \text{Binomial} \left (r+\frac{n}{2}, 1 - \frac{10r}{r+\frac{n}{2}} \right)$ be random variables that denotes the number of samples from $(1 - \frac{10r}{r+\frac{n}{2}})$ mixture component in $Z_m$ and $Z_m'$ respectively. Let $\Omega$ denote the sample space of $Z_m$ and $Z_m'$.
\begin{equation*}
    \begin{split}
        D_{TV}(Z_m, Z_m') &= \max_{E  \subseteq \Omega} \ \lvert \Pr(Z_m \in E) - \Pr(Z_m' \in E)\rvert \\
        &= \max_{E  \subseteq \Omega} \ \lvert \Pr \left(Z_m \in E \mid Y \leq \frac{n}{2}\right)\Pr \left(Y \leq \frac{n}{2}\right) + \Pr \left(Z_m \in E \mid Y > \frac{n}{2}\right)\Pr \left(Y > \frac{n}{2} \right)\\    & \hspace{34pt} -\Pr \left(Z_m' \in E \mid Y' \leq \frac{n}{2} \right)\Pr \left(Y' \leq \frac{n}{2}\right) - \Pr \left(Z_m' \in E \mid Y' > \frac{n}{2} \right)\Pr \left(Y' > \frac{n}{2}\right) \rvert \\
    \end{split}
\end{equation*}
Since $Y$ and $Y'$ have the same distribution, we have $\Pr\left(Y' \leq \frac{n}{2}\right) = \Pr \left(Y \leq \frac{n}{2} \right)$, and $\Pr \left( Y' > \frac{n}{2} \right) = \Pr \left(Y > \frac{n}{2} \right)$. This gives us
\begin{equation*}
    \begin{split}
    D_{TV}(Z_m, Z_m') 
        &= \max_{E  \subseteq \Omega} \ \lvert \Pr \left(Z_m \in E \mid Y \leq \frac{n}{2}\right)\Pr \left(Y \leq \frac{n}{2}\right) + \Pr \left(Z_m \in E \mid Y > \frac{n}{2}\right)\Pr \left(Y > \frac{n}{2}\right)\\    & \hspace{34pt} -\Pr \left(Z_m' \in E \mid Y' \leq \frac{n}{2}\right)\Pr \left(Y \leq \frac{n}{2}\right) - \Pr \left(Z_m' \in E \mid Y' > \frac{n}{2}\right)\Pr \left(Y > \frac{n}{2}\right) \rvert \\
        &\leq \max_{E  \subseteq \Omega} \Pr \left(Y \leq \frac{n}{2}\right) \mid \Pr \left(Z_m \in E | Y \leq \frac{n}{2} \right) - \Pr \left(Z_m' \in E | Y' \leq \frac{n}{2} \right) \mid \\ 
        & \hspace{34pt} + \Pr \left(Y > \frac{n}{2}\right) \mid  \Pr \left(Z_m \in E | Y > \frac{n}{2}\right) - \Pr \left(Z_m' \in E | Y' > \frac{n}{2}\right) \mid. \\
    \end{split}
\end{equation*}
where the last inequality holds because of the triangle inequality.
Now, note that $\Pr(Z_m \in E | Y \leq \frac{n}{2}) = \Pr(Z_m' \in E | Y' \leq \frac{n}{2})$ for all $E$, and $\lvert  \Pr(Z_m \in E | Y > \frac{n}{2}) - \Pr(Z_m' \in E | Y' > \frac{n}{2}) \rvert \leq 1$. This gives us 
$$D_{TV}(Z_m, Z_m') \leq \Pr \left(Y > \frac{n}{2}\right).$$
We know $\E[Y] = \frac{n}{2}-9r$, and $\Var[Y] = \left(\frac{n}{2}+r \right)\left(1 - \frac{10r}{\frac{n}{2}+r}\right) \left(\frac{10r}{\frac{n}{2}+r}\right) \leq 10r$. Using Bernstein's inequality, we get 
\begin{equation*}
    \begin{split}
        \Pr \left[Y > \frac{n}{2} \right] &= \Pr(Y - \E[Y] > 9r)\\
        &\leq \exp \left(\frac{-(9r)^2}{2(10r + 9r/3)}\right)\\
        &\leq \exp\left(\frac{-81r}{26}\right).
    \end{split}
\end{equation*}
So we get $D_{TV}(Z_m, Z_m') \leq \exp\left(\frac{-81r}{26}\right)$.

Next, we calculate $D_{TV}(X_m, Z_m')$. We write $X_m = (X_m^1, X_m^2)$ and $Z_m' = (Z_m^{1'}, Z_m^{2'})$ where $X_m^1$ and $Z_m^{1'}$ denote the first $\frac{n}{2}$ samples of $X_m$ and $Z_m'$ , and $X_m^2$ and $Z_m^{2'}$ denote rest of their samples. Since $X_m^1$ and $Z_m^{1'}$ are drawn from the same distribution, $\Pi_{i=1}^\frac{n}{2} N(\mu, I)$, and $Z_m^{1'}, X_m^1, X_m^2$ are independent, we get $(Z_m^{1'}, X_m^2)$ and $(X_m^{1}, X_m^2)$ are equal in distribution. This gives us
\begin{align*}
D_{TV}(X_m, Z_m') = D_{TV}((X_m^1, X_m^2), (Z_m^{1'}, Z_m^{2'})) = D_{TV}((Z_m^{1'}, X_m^2), (Z_m^{1'}, Z_m^{2'})).
\end{align*}
From Lemma \ref{lem:coupling}, we know that, if with probability at least $1 - \epsilon_1$ over  $Z_m^{1'}$, $D_{TV}(X_m^2|Z_m^{1'}, Z_m^{2'}|Z_m^{1'}) \leq \epsilon_2$, then $D_{TV}((Z_m^{1'}, X_m^2), (Z_m^{1'}, Z_m^{2'})) \leq \epsilon_1 + \epsilon_2$. Here, $Z_m^{1'}$ and $X_m^2$ are independent, and the only dependency between $Z_m^{1'}$ and $Z_m^{2'}$ is via the mean $\tilde{\mu}$ of the elements of $Z_m^{1'}$. So $D_{TV}(X_m^2|Z_m^{1'}, Z_m^{2'}|Z_m^{1'}) = D_{TV}(X_m^2, Z_m^{2'}|\tilde{\mu})$. We will show that with high probability over $\tilde{\mu}$, this total variation distance is small. \\

We first estimate $\norm{\tilde{\mu} - \mu}$. Note that $\E_{Z_m^{1'}}[\norm{\tilde{\mu} - \mu}^2] = \frac{2d}{n}$, and $\frac{n}{2}\norm{\tilde{\mu} - \mu}^2$ is a $\chi^2$ random variable with $d$ degrees of freedom.  To bound the deviation of $\norm{\tilde{\mu} - \mu}^2$ around it's mean, we will use the following concentration bound for a $\chi^2$ random variable $R$ with $d$ degrees of freedom \citep[Example 2.5]{wainwright2015basic}.
\begin{align*}
\Pr[|R - d| \ge dt] \le 2e^{-dt^2/8}, \text{ for all } t \in (0,1).
\end{align*}
This gives us $\Pr(\lvert \frac{n}{2}\norm{\tilde{\mu} - \mu}^2  - d \rvert  \geq {0.5 d}) \leq 2e^{-d/32}$, that is, $\norm{\tilde{\mu} - \mu} \leq \sqrt{\frac{3 d}{n}} \leq \sqrt{3 \epsilon \log d}$ with probability at least $1 - 2e^{-d/32}$.

$X_m^2$ is distributed as the product of $\frac{n}{2}+r$ gaussiaus $\Pi_{i = 1}^{\frac{n}{2}+r} N(\mu, I_{d \times d})$ and $Z_m^{2'}|\tilde{\mu}$ is distributed as the product of $\frac{n}{2}+r$ mixture distributions  $\Pi_{i=1}^{\frac{n}{2}+r} (1 - \frac{10r}{\frac{n}{2}+r})N(\mu, I_{d \times d}) + \frac{10r}{\frac{n}{2}+r}N(\tilde{\mu}, I_{d \times d})$. We evaluate the total variation distance between these two distributions by bounding their squared Hellinger distance, since squared Hellinger distance is easy to bound for product distributions and is within a quadratic factor of the total variation distance for any distribution. By the subadditivity of the squared Hellinger distance, we get 
\begin{equation}\label{eq:thm-hellinger-subadd}
    \begin{split}
        &H\left(\Pi_{i = 1}^{\frac{n}{2}+r} N(\mu, I_{d \times d}),\Pi_{i=1}^{\frac{n}{2}+r} \left(1 - \frac{10r}{\frac{n}{2}+r}\right)N \left(\mu, I_{d \times d} \right) + \frac{10r}{\frac{n}{2}+r}N \left(\tilde{\mu}, I_{d \times d} \right) \right)^2\\
        &\leq \left(\frac{n}{2}+r \right) \ H\left( N \left(\mu, I_{d \times d}\right), \left(1 - \frac{10r}{\frac{n}{2}+r}\right)N \left(\mu, I_{d \times d}\right) + \frac{10r}{\frac{n}{2}+r}N \left(\tilde{\mu}, I_{d \times d}\right) \right)^2.
    \end{split}
\end{equation}
For sufficiently large $d$, $r$ and $n$ satisfy $r \leq \frac{n}{18}$, so we can use Lemma \ref{lem:hellinger-gaussian} to get
\begin{equation}\label{eq:thm-hellinger-ubound}
    \begin{split}
         H\left( N(\mu, I_{d \times d}), \left(1 - \frac{10r}{\frac{n}{2}+r} \right)N \left(\mu, I_{d \times d}\right) + \frac{10r}{\frac{n}{2}+r}N \left(\tilde{\mu}, I_{d \times d}\right) \right)^2
         &\leq \frac{576r^2}{n^2}e^{3\norm{\tilde{\mu} - \mu}^2}\\
         &\leq \frac{576r^2 d^{9\epsilon}}{n^2},
    \end{split}
\end{equation}
with probability at least $1 - 2e^{-d/32}$ over $\tilde{\mu}$.  From \eqref{eq:thm-hellinger-subadd} and \eqref{eq:thm-hellinger-ubound}, we get that with probability at least $1 - 2e^{-d/32}$ over $\tilde{\mu}$,
\begin{equation*}
    \begin{split}
        &H\left(\Pi_{i = 1}^{\frac{n}{2}+r} N(\mu, I_{d \times d}),\Pi_{i=1}^{\frac{n}{2}+r} \left(1 - \frac{10r}{\frac{n}{2}+r} \right) N(\mu, I_{d \times d}) + \frac{10r}{\frac{n}{2}+r}N(\tilde{\mu}, I_{d \times d}) \right)^2\leq (\frac{n}{2}+r)\frac{576r^2 d^{9\epsilon}}{n^2}\leq \frac{576r^2 d^{9\epsilon}}{n},
    \end{split}
\end{equation*}
where the last inequality holds because $r < \frac{n}{2}$. As the total variation distance between two distributions is upper bounded by $\sqrt{2}$ times their Hellinger distance, we get that with probability at least $1 - 2e^{-d/32}$ over $\tilde{\mu}$,
\begin{equation*}
    \begin{split}
        &D_{TV}\left(\Pi_{i = 1}^{\frac{n}{2}+r} N(\mu, I_{d \times d}),\Pi_{i=1}^{\frac{n}{2}+r} \left(1 - \frac{10r}{\frac{n}{2}+r}\right)N(\mu, I_{d \times d}) + \frac{10r}{\frac{n}{2}+r}N(\tilde{\mu}, I_{d \times d}) \right)\\
        &\leq \frac{24 \sqrt{2} r d^{9\epsilon/2}}{\sqrt{n}}\leq \frac{24\sqrt{2} r n^{9\epsilon}}{\sqrt{n}},
    \end{split}
\end{equation*}
where the last inequality is true because $n > \sqrt{d}$.

Now, from Lemma \ref{lem:coupling}, we know that if with probability at least $1 - \epsilon_1$ over  $Z_m^{1'}$, $D_{TV}(X_m^2|Z_m^{1'}, Z_m^{2'}|Z_m^{1'}) \leq \epsilon_2$, then $D_{TV}((Z_m^{1'}, X_m^2), (Z_m^{1'}, Z_m^{2'})) \leq \epsilon_1 + \epsilon_2$. In this
case, $\epsilon_1 = 2e^{-d/32}$ and $\epsilon_2 = \frac{24\sqrt{2} r n^{9\epsilon}}{\sqrt{n}}$, so we get $D_{TV}((Z_m^{1'}, X_m^2), (Z_m^{1'}, Z_m^{2'})) =  D_{TV}(X_m, Z_m') \leq 2e^{-d/32} + \frac{24\sqrt{2} r n^{9\epsilon}}{\sqrt{n}}$. We also know that $D_{TV}(Z_m, Z_m') \leq e^{-81r/26}$. Using triangle inequality, we get 
\begin{align*}
    D_{TV}(X_m, Z_m) \leq 2e^{-d/32} + \frac{24\sqrt{2} r n^{9\epsilon}}{\sqrt{n}} + e^{-81r/26}.
\end{align*}
 For $\delta > 2(2e^{-d/32} + e^{-81r/26})$, and for $  r \leq \frac{n ^{\frac{1}{2} - 9\epsilon} \delta }{48\sqrt{2}}$, we get $D_{TV}(X_m, Z_m) \leq \delta$. For $d$ large enough, setting $\delta= \frac{1}{3}$ and $r \leq \frac{n ^{\frac{1}{2} - 9\epsilon} }{144\sqrt{2}} $, we get the desired result. Note that we haven't tried to optimize the constants in this proof.

\begin{lemma}\label{lem:hellinger-gaussian}
Let $P = N(0, I_{d \times d})$ and $Q = N(\hat{\mu}, I_{d \times d})$ be $d$-dimensional gaussian distributions. For  $r \leq \frac{n}{18}$, $H\left(P, \left(1 - \frac{10r}{r+\frac{n}{2}}\right)P + \frac{10r}{r+\frac{n}{2}}Q\right) \leq  \frac{24r}{n}e^{\frac{3\norm{\hat{\mu}}^2}{2}}$.
\end{lemma}
\begin{proof}
We work in the rotated basis where $Q = N((\norm{\hat{\mu}},\underbrace{0, 0, \dots, 0}_{d-1 \text{ times}}), I_{d \times d})$ and $P = N(0, I_{d \times d})$. Let $P_1 = N(0, 1)$ and $Q_1 = N(\norm{\hat{\mu}}, 1)$ denote the projection of $P$ and $Q$ along the first coordinate axis respectively. Note that the mixture distribution in question is the product of $\left(\left(1 - \frac{10r}{r+\frac{n}{2}}\right)P_1 + \frac{10r}{r+\frac{n}{2}}Q_1\right)$ and $N(0, I_{d-1 \times d-1})$, and $P$ is the product of $P_1$ and $N(0, I_{d-1 \times d-1})$. Since the squared Hellinger distance is subadditive for product distributions, we get,
\begin{equation*}
    \begin{split}
        H\left(P, \left(1 - \frac{10r}{r+\frac{n}{2}}\right)P + \frac{10r}{r+\frac{n}{2}}Q\right)^2 &\leq H\left(P_1,\left(1 - \frac{10r}{r+\frac{n}{2}}\right)P_1 + \frac{10r}{r+\frac{n}{2}}Q_1\right)^2 + H(N(0,I_{d-1 \times d-1}), N(0,I_{d-1 \times d-1}))^2 \\
        &= H\left(P_1,\left(1 - \frac{10r}{r+\frac{n}{2}}\right)P_1 + \frac{10r}{r+\frac{n}{2}}Q_1\right)^2.
    \end{split}
\end{equation*}

Therefore, to bound the required Hellinger distance, we just need to bound $H \left(P_1, \left(1 - \frac{10r}{r+\frac{n}{2}}\right)P_1 + \frac{10r}{r+\frac{n}{2}}Q_1 \right)$. Let $p_1$ and $q_1$ denote the probability densities of $P_1$ and $\left(\left(1 - \frac{10r}{r+\frac{n}{2}}\right)P_1 + \frac{10r}{r+\frac{n}{2}}Q_1\right)$ respectively. We get $H\left(P_1,\left(1 - \frac{10r}{r+\frac{n}{2}}\right)P_1 + \frac{10r}{r+\frac{n}{2}}Q_1\right)^2 = \int_{- \infty}^{\infty} \left(\sqrt{p_1} - \sqrt{q_1}\right )^2 dx$
\begin{equation*}
    \begin{split}
      &= \int_{- \infty}^{\infty} \left(\sqrt{\frac{1}{\sqrt{2 \pi}}e^{-x^2/2}}  - \sqrt{\left(1 -\frac{10r}{r+\frac{n}{2}}\right)\frac{1}{\sqrt{2 \pi}} e^{-x^2/2} +  \frac{10r}{r+\frac{n}{2}} \frac{1}{\sqrt{2 \pi}} e^{-(x  - \norm{\hat{\mu}})^2/2}}\right)^2 dx\\
      &= \int_{- \infty}^{\infty}
      \frac{1}{\sqrt{2 \pi}} e^{-x^2/2}
      \left(1 - \sqrt{1 - \frac{10r}{r+\frac{n}{2}} + \frac{10r}{r+\frac{n}{2}}e^{\frac{- \norm{\hat{\mu}}^2 + 2 \norm{\hat{\mu}} x}{2}}}  \right)^2 dx.\\
    \end{split}
\end{equation*}
We will evaluate this integral as a sum of integral in two regions. 
\begin{enumerate}
    \item From $-\infty$ to $\nmu/2$:
\begin{equation*}
    \begin{split}
    \int_{- \infty}^{\nmu/2}
      \frac{1}{\sqrt{2 \pi}} e^{-x^2/2}
      \left(1 - \sqrt{1 - \frac{10r}{r+\frac{n}{2}} + \frac{10r}{r+\frac{n}{2}}e^{\frac{- \norm{\hat{\mu}}^2 + 2 \norm{\hat{\mu}} x}{2}}}  \right)^2 dx
      &\leq 
      \int_{- \infty}^{\nmu/2}
      \frac{1}{\sqrt{2 \pi}} e^{-x^2/2}
      \left(1 - \sqrt{1 - \frac{10r}{r+\frac{n}{2}} }  \right)^2 dx.
    \end{split}
\end{equation*}
Since $r \leq \frac{n}{18}$, we get $\frac{10r}{r + \frac{n}{2}} \leq 1$.  Using $1 - y \leq \sqrt{1 - y}$ for $0 \leq y \leq 1$, we get 
\begin{equation*}
    \begin{split}
      \int_{- \infty}^{\nmu/2}
      \frac{1}{\sqrt{2 \pi}} e^{-x^2/2}
      \left(1 - \sqrt{1 - \frac{10r}{r+\frac{n}{2}} }  \right)^2 dx
      &\leq 
      \int_{- \infty}^{\nmu/2}
      \frac{1}{\sqrt{2 \pi}} e^{-x^2/2}
      \left(\frac{10r}{r+\frac{n}{2}}   \right)^2 dx\\
      &\leq \frac{400r^2}{n^2}.
    \end{split}
\end{equation*}

\item From $\frac{\nmu}{2}$ to $\infty$, we get $\int_{\nmu/2}^{\infty}
      \frac{1}{\sqrt{2 \pi}} e^{-x^2/2}
      \left(1 - \sqrt{1 - \frac{10r}{r+\frac{n}{2}} + \frac{10r}{r+\frac{n}{2}}e^{\frac{- \norm{\hat{\mu}}^2 + 2 \norm{\hat{\mu}} x}{2}}}  \right)^2 dx$.
\begin{equation*}
    \begin{split}
      &\leq
      \int_{\nmu/2}^{\infty}
      \frac{1}{\sqrt{2 \pi}} e^{-x^2/2}
      \left( \sqrt{1 + \frac{10r}{r+\frac{n}{2}}  e^{\frac{- \norm{\hat{\mu}}^2 + 2 \norm{\hat{\mu}} x}{2}}}  - 1  \right)^2 dx.\\
    \end{split}
\end{equation*}
This is because $x \geq \nmu/2$, and therefore $\frac{10r}{r+\frac{n}{2}}e^{\frac{- \norm{\hat{\mu}}^2 + 2 \norm{\hat{\mu}} x}{2}} \geq \frac{10r}{r+\frac{n}{2}}$. Now, using $\sqrt{1+y} \leq 1 + \frac{y}{2}$, we get

  \begin{equation*}
    \begin{split}
&\int_{\nmu/2}^{\infty}
      \frac{1}{\sqrt{2 \pi}} e^{-x^2/2}
      \left( \sqrt{1 + \frac{10r}{r+\frac{n}{2}}  e^{\frac{- \norm{\hat{\mu}}^2 + 2 \norm{\hat{\mu}} x}{2}}}  - 1  \right)^2 dx\\
      &\leq
      \int_{\nmu/2}^{\infty}
      \frac{1}{\sqrt{2 \pi}} e^{-x^2/2}
      \left( {1 + \frac{5r}{r+\frac{n}{2}}  e^{\frac{- \norm{\hat{\mu}}^2 + 2 \norm{\hat{\mu}} x}{2}}}  - 1  \right)^2 dx\\
      &\leq
     \frac{100r^2}{n^2} \int_{\nmu/2}^{\infty}
      \frac{1}{\sqrt{2 \pi}} {e^{- \norm{\hat{\mu}}^2 + 2 \norm{\hat{\mu}} x}} e^{-x^2/2}
          dx\\
      &=
    \frac{100r^2}{n^2}e^{-\nmu^2} \int_{\nmu/2}^{\infty}
      \frac{1}{\sqrt{2 \pi}} {e^{2\nmu x - x^2/4}} e^{-x^2/4}
          dx.\\
    \end{split}
\end{equation*}

Since $2\nmu x - x^2/4 \leq 4 \nmu^2$, we get 
      
\begin{equation*}
    \begin{split}
    \frac{100r^2}{n^2} e^{-\nmu^2}\left(
      \int_{\nmu/2}^{\infty} 
      e^{2\nmu x - x^2/4}
      \frac{1}{\sqrt{2 \pi}} e^{-x^2/4}
      \right) dx
      &\leq
      \frac{100r^2}{n^2} e^{3\nmu^2}\left(
      \int_{\nmu/2}^{\infty} 
      \frac{1}{\sqrt{2 \pi}} e^{-x^2/4}
      \right) dx\\
      &\leq
      \frac{100\sqrt{2} r^2}{n^2} e^{3\nmu^2}\left(
      \int_{-\infty}^{\infty} 
      \frac{1}{\sqrt{4 \pi}} e^{-x^2/4}
      \right) dx\\
      &\leq 
      \frac{100\sqrt{2} r^2}{n^2} e^{3\nmu^2}.\\
    \end{split}
\end{equation*}

\end{enumerate}

Adding the two integrals, we get 
\begin{equation*}
    \begin{split}
        H \left(P_1, \left(1 - \frac{10r}{r+\frac{n}{2}}\right)P_1 + \frac{10r}{r+\frac{n}{2}}Q_1 \right)^2
        &\leq \frac{400r^2}{n^2} + \frac{100\sqrt{2}r^2}{n^2}e^{3\nmu^2}\\
        &\leq \frac{576r^2}{n^2}e^{3\nmu^2}.
    \end{split}
\end{equation*}
This gives us $H(P, \left(1 - \frac{10r}{r+\frac{n}{2}}\right)P + \frac{10r}{r+\frac{n}{2}}Q) \leq \frac{24r}{n} e^{3\nmu^2/2}$ which completes the proof.
\end{proof}
\end{proof}

\subsection{Lower Bound for Procedures which Return a Superset of the Input Samples}
In this section we prove the lower bound from Proposition \ref{prop:gaussian_modify_full}.

\begin{proposition}
Let $\mathcal{C}$ denote the class of $d-$dimensional Gaussian distributions $N\left(\mu, I\right)$ with unknown mean $\mu$. There is an absolute constant, $c$, such that for sufficiently large $d$, if $n \le \frac{cd}{\log d},$ there is no $(n,n+1)$ amplification procedure that always returns a superset of the original $n$ points.
\end{proposition}

%\begin{theorem}\label{thm:gaussian_lower_modify}
%Let $\mathcal{C}$ denote the class of $d-$dimensional gaussian distributions with unknown mean $\mu$ and identity covariance, $N\left(\mu, I\right)$. Then for sufficiently large $d$ and $n< O(d/\log d)$, $\mathcal{C}$ does not admit an $(n,n+1)$ amplification procedure which does not modify the input samples.
%\end{theorem}
\begin{proof}
The outline of the proof is very similar to the proof of Proposition \ref{prop:gaussian_lower}. As in the proof of Proposition \ref{prop:gaussian_lower}, we define a verifier $v(Z_{n+1})$ for the distribution $N(\mu,I)$ which takes as input $(n+1)$ samples $\{x_i'\in \mathbb{R}^d, i \in [n+1]\}$, and a distribution $D_{\mu}$ over $\mu$, such that if $n< O(d/\log(d))$; (i) for all $\mu$, the verifier will accept with probability $1-1/e^2$ when given as input a set $Z_{n+1}$ of $(n+1)$ i.i.d. samples from $N(\mu,I)$, (ii) but will reject any $(n,n+1)$ amplification procedure which does not modify the input samples with probability $1-1/e^2$, where the probability is with respect to the randomness in $\mu\leftarrow D_{\mu}$, the set $X_n$ and in any internal randomness of the amplifier. Note that by Definition \ref{def2} of an amplification procedure, this implies that there is no $(n,n+1)$ amplification procedure which does not modify the input samples for $n< O(d/\log(d))$. We choose $D_{\mu}$ to be $N(0,\sqrt{d}I)$. Let $\hat{\mu}_{-i}$ be the mean of the all except the $i$-th sample returned by the amplification procedure. The verifier performs the following tests, and accepts if all tests pass, and rejects otherwise---
\begin{enumerate}
\item $\forall \;i \in [n+1],\norm{x_{i}'-\mu}^2\le 15d$.
    \item $\forall\; i \in [n+1], \ip{x_{i}'-\hat{\mu}_{-i}}{\mu-\hat{\mu}_{-i}} \ge d/(4n)$.
\end{enumerate}
We first show that for a sufficiently large constant $C$ and $n< O(d/\log(d))$, $(n+1)$ i.i.d. samples from $N(\mu, I)$ pass the above tests with probability at least $1-1/e^2$. As $\norm{x_{i}'-\mu}^2$ is a  $\chi^2$ random variable with $d$ degrees of freedom, by the concentration bound for a $\chi^2$ random variable \eqref{eq:chisquare}, a true sample $x_{i}'$ passes the first test with failure probability $e^{-3d}$. Hence by a union bound, all samples $\{x_i, i \in [n+1]\}$ pass the first test with probability at least $1-de^{-3d}\ge 1- 1/e^3$. 
Let $E$ denote the following event, 
\begin{align*}
\forall\; i\in[n+1], \norm{\hat{\mu}_{n}-\mu}^2 \ge d/n-\sqrt{20d\log d}/n &\ge d/(2n),\\
\forall\; i\in[n+1],
\norm{\hat{\mu}_{n}-\mu}^2 \le d/n+\sqrt{20d\log d}/n &\le 2d/n.
\end{align*}
Note that $\hat{\mu}_{-i}\leftarrow N(\mu,\frac{I}{n})$. Hence, by using \eqref{eq:chisquare2} with $t=20\sqrt{\frac{\log d}{ d}}$, and a union bound over all $i\in [n+1]$, 
\begin{align*}
\Pr[E] \ge 1-1/e^3.
\end{align*}
Note that as $x_{i}'\leftarrow N(\mu,I)$, for a fixed $\hat{\mu}_{-i}$, $\ip{x_{i}'-\hat{\mu}_{-i}}{\mu-\hat{\mu}_{-i}} \leftarrow N(\norm{\hat{\mu}_{-i}-\mu}^2, \norm{\hat{\mu}_{-i}-\mu}^2)$. Hence conditioned on $E$, by standard Gaussian tail bounds,
\begin{align*}
\Pr\Big[ \ip{x_{i}'-\hat{\mu}_{-i}}{\mu-\hat{\mu}_{-i}} \le d/(2n) - \sqrt{20d\log d/n}  \Big] \le 1/n^2, \\
\implies \Pr\Big[ \ip{x_{i}'-\hat{\mu}_{-i}}{\mu-\hat{\mu}_{-i}} \le d/(4n) \Big] \le 1/n^2,
\end{align*}
where in the last step we use the fact that $n< \frac{d}{C\log d}$ for a large constant $C$. Therefore, conditioned on  $E$, $\{x_i, i \in [n+1]\}$ pass the third test with probability at least $1-1/e^3$. Hence by a union bound, $(n+1)$ samples drawn from $N(\mu,I)$ will satisfy all 3 tests with failure probability at most $1/e^2$. Hence for any $\mu$, the verifier accepts $n+1$ i.i.d. samples from $N(\mu,I)$ with probability at least $1-1/e^2$.\\

We now show that for $n<\frac{d}{C\log d}$ and $\mu$ sampled from $ D_{\mu}=N(0,\sqrt{d}I)$, the verifier rejects any $(n,n+1)$ amplification procedure which does not modify the input samples with high probability over the randomness in $\mu$ and the set $X_n$. Let $D_{\mu|X_n}$ be the posterior distribution of $\mu$ conditioned on the set $X_n$. As in Proposition \ref{prop:gaussian_lower}, $D_{\mu|X_n}=N(\bar{\mu},\bar{\sigma}^2 I)$, where,
\begin{align*}
\bar{\mu} = \frac{n}{n+1/\sqrt{d}}\mu_0, \quad \bar{\sigma}^2 = \frac{1}{n+1/\sqrt{d}}.
\end{align*}
We will show that with probability $1-e^{-3d}$ over the randomness in the set $X_n$ received by the amplifier and with probability $1-1/e^2$ over $\mu\leftarrow D_{\mu|X_n}$ and any internal randomness of the amplifier, the amplifier cannot output a set $Z_{n+1}$ which contains the set $X_n$ as a subset and which is accepted by the verifier. To show this, we first claim that $\norm{\mu_0}\le 30d^{3/4}$ with probability $1-e^{d}$. Note that $\mu_0 \leftarrow N(\mu, \frac{I}{n})$, where $\mu\leftarrow N(0,\sqrt{d}I)$.  By \eqref{eq:chisquare}, with probability at least $1-e^{-3d}$, $\norm{\mu}\le 15d^{3/4}$ and $\norm{\mu-\mu_0}\le 15\sqrt{d}$. Hence by the triangle inequality, $\norm{\mu_0}\le 30d^{3/4}$ with probability at least $1-e^{-3d}$. We now show that for sets $X_n$ such that $\norm{\mu_0}\le 30d^{3/4}$, $Z_{n+1}$ cannot pass the verifier with probability more than $1-e^2$ over the randomness in $\mu|X_n$. The proof consists of two cases, and the analysis of the cases is similar to the proof of Proposition \ref{prop:gaussian_lower}. Without loss of generality, assume that $Z_{n+1}=\{x_1',X_n\}$, hence $x_1'$ is the only sample not present in the set. We will show that either $x_1'$ or $\hat{\mu}_{-1}$ fail one of the three tests performed by the verifier with high probability.

\subsubsection*{Case 1: $\norm{x_{1}'-\bar{\mu}}^2 \ge 100d$.}

We show that the first test is not satisfied with high probability in this case. As $\mu|X_n \leftarrow N(\bar{\mu},\bar{\sigma}^2)$, hence by \eqref{eq:chisquare}, $ \norm{\mu-\bar{\mu}}^2 \le 15d/n $ with probability $1-e^{-3d}$. Therefore, if  $\norm{x_{1}'-\bar{\mu}}^2 \ge 100d$, then with probability $e^{-3d}$,
\begin{align*}
\norm{x_{1}'-{\mu}}^2 \ge (\sqrt{100d} - \sqrt{15d/n})^2 > 15d,
\end{align*}
in which case the first test is not satisfied. Hence in the first case, the amplifier succeeds with probability at most $e^{-3d}$.

\subsubsection*{Case 2: $\norm{x_{1}'-\bar{\mu}}^2 < 100d$.}

Note that for the sample $x_1'$, $\mu_{-1}=\mu_0$ as the last $n$ samples are the same as the original set $X_n$. We now bound $\norm{\hat{\mu}_{-1}-\bar{\mu}}$ as follows,
\begin{align*}
\norm{\hat{\mu}_{-1}-\bar{\mu}} = \Big\| \mu_0-\frac{n}{n+1/\sqrt{d}}\mu_0 \Big\| \le  \frac{\norm{\mu_0}}{n\sqrt{d}}\le \frac{30d^{1/4}}{n}.
\end{align*}
We now expand $\ip{x_1'-\hat{\mu}_{-1}}{\mu-\hat{\mu}_{-1}}$ in the third test as follows,
\begin{align*}
\ip{x_1'-\hat{\mu}_{-1}}{\mu-\hat{\mu}_{-1}}  &= \ip{x_1'-\bar{\mu}}{\mu-\bar{\mu}}  - \ip{\hat{\mu}_{-1}-\bar{\mu}}{\mu-\bar{\mu}} - \ip{x_1'-\bar{\mu}}{\hat{\mu}_{-1}-\bar{\mu}}+ \norm{\hat{\mu}_{-1}-\bar{\mu}}^2,\\
&\le \ip{x_1'-\bar{\mu}}{\mu-\bar{\mu}}  - \ip{\hat{\mu}_{-1}-\bar{\mu}}{\mu-\bar{\mu}} + \norm{x_1'-\bar{\mu}} \norm{\hat{\mu}_{-1}-\bar{\mu}}+ \norm{\hat{\mu}_{-1}-\bar{\mu}}^2.
\end{align*}
Note that $\ip{\hat{\mu}_{-1}-\bar{\mu}}{{\mu}-\bar{\mu}}$ is distributed as $N(0,\bar{\sigma}^2\norm{\hat{\mu}_{-1}-\bar{\mu}}^2)$ and hence with probability $1-1/e^3$ it is at most $10\norm{\hat{\mu}_{-1}-\bar{\mu}}/\sqrt{n}$. Similarly, with probability $1-1/e^3$, $\ip{x_1'-\bar{\mu}}{\mu-\bar{\mu}}$ is at most $10\norm{x_1'-\bar{\mu}}/\sqrt{n}$. Therefore, with probability $1-2/e^3$,
\begin{align*}
\ip{x_1'-\hat{\mu}_{-1}}{\mu-\hat{\mu}_{-1}} &\le 10\norm{x_1'-\bar{\mu}}/\sqrt{n}  + 10\norm{\hat{\mu}_{-1}-\bar{\mu}}/\sqrt{n} + \norm{x_1'-\bar{\mu}} \norm{\hat{\mu}_{-1}-\bar{\mu}}+ \norm{\hat{\mu}_{-1}-\bar{\mu}}^2,\\
&\le 100\sqrt{\frac{d}{n}} + 300\frac{d^{3/4}}{n^2} + 300\frac{d^{3/4}}{{n}} + 900\frac{\sqrt{d}}{n^2}\\
&\le 100\sqrt{\frac{d}{n}} + 1500\frac{d^{3/4}}{{n}}\\
&=  100\sqrt{\frac{n}{d}}\Big(\frac{d}{n}\Big) + \frac{1500}{d^{1/4}}\Big(\frac{d}{n}\Big).
\end{align*}
Hence for a sufficiently large constant $C$, $n<\frac{d}{C\log d}$ and $d$ sufficiently large, with probability $1-2/e^3$,
\begin{align*}
\ip{x_1'-\hat{\mu}_{-1}}{\mu-\hat{\mu}_{-1}} &\le \frac{d}{5n},
\end{align*}
which implies that the second test is not satisfied. Hence the amplifier succeeds in this case with probability at most $2/e^3$.

The overall success probability of the amplifier is the maximum success probability across the two cases, hence for sets $X_n$ such that the $\norm{\mu_0}\le 30d^{3/4}$, the verifier accepts the amplified set $Z_{n+1}$ with probability at most $2/e^3$. As $\Pr\Big[\norm{\mu_0}\le 30d^{3/4}\Big]\ge 1-e^{-3d}$, the overall success probability of the amplifier over the randomness in $\mu$, $X_n$ and any internal randomness of the amplifier is at most $1/e^2$.
\end{proof}

%\section{Discrete Distributions with Bounded Support}
\section{Proofs: Discrete Distributions with Bounded Support}
\subsection{Upper Bound}
In this section we prove the upper bound from Theorem \ref{thm:discrete-full}. The algorithm itself is presented in Algorithm \ref{alg:discrete}.
For clarity of writing, we assume that the number of input samples is $4n$, instead of $n$.
\algdiscrete

\begin{proposition}
Let $\mathcal{C}$ denote the class of discrete distributions with support size at most $k$. For sufficiently large $k,$ and $m = 4n+O\left(\frac{n}{\sqrt{k}}\right)$, $\mathcal{C}$ admits an $\left(4n, m\right)$ amplification procedure.
\end{proposition}

%\begin{theorem}\label{discrete-ub}
%Let $\mathcal{C}$ denote the class of discrete distributions with support size at most $k$. For all $m = n+O\left(\frac{n}{\sqrt{k}}\right)$, and $k$ large enough, $\mathcal{C}$ admits an $\left(n, m\right)$ amplification procedure.%\VS{Any condition on $k,n$ for this?}
%\end{theorem}
\begin{proof}
To avoid dependencies between the count of different elements, we first prove our results in a Poissonized setting, and then in lemma \ref{lem:discrete-ub2}, we describe how to use the amplifier for Poissonized setting to get an amplifier for the original multinomial setting. Let $D \in \mathcal{C}$ be an unknown probability distribution over $[k]$, and let $p_i$ denote the probability mass associated with $i \in [k]$. Throughout the proof, we use random variable $X_q$ to denote $q$ independent samples from $D$, where $q$ can also be a random variable. Suppose we are given $N = N_1 + N_2$ independent samples from $D$, denoted by $X_{N_1} \text{ and } X_{N_2}$, where $N_1$ and $N_2$ are drawn from $\poisson(n)$. We show how to amplify them to  $\tilde{M} = N + R$ samples, denoted by $Z_{\tilde{M}}$, such that $D_{TV}(Z_{\tilde{M}}, X_{M})$ is small, where $M \leftarrow \poisson(2n+r)$.

Our amplifying procedure involves estimating the probability of each element using $X_{N_1}$, generating $R$ independent samples using these estimates, and randomly shuffling these samples with $X_{N_2}$.
Let $u_i$ be the count of element $i$ in $X_{N_1}$ and $y_i$ be the count of $i$ in $X_{N_2}$ noting they are both distributed as  $\poisson(np_i)$. The amplification procedure proceeds through the following  steps:
\begin{enumerate}
    \item Estimate the frequency $\hat{p}_i$ of each element using $u_i$, that is,  $\hat{p}_i = \frac{u_i}{n}$.
    \item Draw $\hat{z}_i \leftarrow \poisson(r\hat{p}_i) $ additional samples of element $i$ for all $i \in [k]$.
    \item Append these generated samples to $X_{N_2}$ to get $Z_{N_2 + R}$.
    \item Randomly permute the elements of $Z_{N_2 + R}$, and append them to $X_{N_1}$ to get $Z_{\tilde{M}}$. 
    %\textit{Note: when we leave the poissonized setting we will add in the samples thrown out by the poissonization to ensure that our amplifier produces sufficiently many samples}.
\end{enumerate}

We first show that $Z_{\tilde{M}}$ is close in total variation distance, to $\poisson(2n+r)$ samples generated from $D$. We will prove this by showing that with high probability over the choice of $X_{N_1}$, the distribution of $Z_{N_2+R}$ is close to $\poisson(n+r)$ samples generated from $D$. After this, we can use lemma \ref{lem:coupling} to show that appending $Z_{N_2 + R}$ to the samples in $X_{N_1}$ results in a sequence with low total variation distance to $X_{M}$. Since our amplification procedure randomly permutes the last $N_2 + R$ elements, we can argue this using only the count of each element. Recall $y_i$ is the count of element $i$ in $X_{N_2}$, and $\hat{z}_i$ is the number of additional samples of element $i$ added by our amplification procedure. Let $z_i \leftarrow \poisson(rp_i)$, and let  $v_i=y_i+z_i$ and $\hat{v}_i=y_i+\hat{z}_i$. Here, $v_i$ denotes the count of element $i$ in $\poisson(n+r)$ samples drawn from $D$, and $\hat{v}_i$ denotes the corresponding count in samples generated using our amplification procedure. We use $P_v$ to denote the distribution associated with random variable $v$.

\begin{lemma}\label{lem:discrete1}
For $r\le n\epsilon^{1.5}/(4\sqrt{k})$, with probability $1-\eps$ over the randomness in $\{u_i,i\in [k]\}$, 
\begin{align*}
d_{TV} \left (\prod_{i=1}^{k}{v_i}, \prod_{i=1}^{k}{\hat{v}_i} \right)  \le \eps/2.
\end{align*}
where $\prod$ refers to the product distribution. 
% Used to be:
%d_{TV} \left (\prod_{i=1}^{k}v_i, \prod_{i=1}^{k}\hat{v}_i \right)  \le \eps/2.
\end{lemma}
\begin{proof}
We partition the support $[k]$ into two sets. Let $S=\{i:p_i\ge \eps/(2nk)\}$ and $S^c= [k]\backslash S$. Let $|S|=k'$. Without loss of generality, assume that $S={\{i:1\le i \le k'\}}$ and $S^c={\{i: k'+1\le i \le k\}}$. We will separately bound the contribution of the variables in the set $S$ and $S^c$  to the total variation distance. For the first set $S$, we will upper bound $\sum_{i=1}^{k'} D_{KL}(v_i\parallel \hat{v}_i)$, and use Pinsker's inequality to then bound the total variation distance. For the second set $S^c$, we will directly bound $\sum_{i=k'+1}^{k} d_{TV}(v_i, \hat{v}_i)$. All our bounds will be with high probability over the randomness in the first set $\{u_i, i \in [k]\}$. 

We first bound the total variation distance for the variables in the first set $S$. Note that because the sum of two Poisson random variables is a Poisson random variable, $v_i$ is distributed as $\poisson(np_i+rp_i)$ and $\hat{v}_i$ is distributed as $\poisson(np_i +ru_i /n)$. We will use the following expression for the KL divergence $D_{KL}(P\parallel Q)$ between two Poisson distributions $P$ and $Q$ with means $\lambda_1$ and $\lambda_2$ respectively---
\begin{align}
    D_{KL}(P\parallel Q) = \lambda_1 \log \left( \frac{\lambda_1}{\lambda_2} \right) + \lambda_2-\lambda_1.
\end{align}
Using this expression, we can write the KL divergence between the distributions of $v_i$ and $\hat{v}_i$ as follows,
\begin{align*}
    D_{KL}({v_i}\parallel {\hat{v}_i}) = p_i (n+r) \log \left( \frac{p_i (n+r)}{p_i n + ru_i/n} \right) + (ru_i/n-rp_i).
\end{align*}
Let $\delta_i = u_i-np_i$. We can rewrite the above expression as follows,
\begin{align*}
    D_{KL}({v_i}\parallel {\hat{v}_i}) &= p_i (n+r) \log \left( \frac{p_i (n+r)}{p_i (n+r) + r\delta_i/n} \right) + r\delta_i/n,\\
     &= p_i (n+r) \log \left( \frac{1}{1 + {r\delta_i}/({np_i(n+r)})} \right) + r\delta_i/n.
\end{align*}
Note that $\log(1+x)\ge x-2x^2$ for $x\ge 0.8$. As $\delta_i\ge -np_i$, therefore ${r\delta_i}/({np_i(n+r)})\ge -0.8$ for $r\le n$. Therefore,
\begin{align}
    p_i (n+r) \log \left( \frac{1}{1 + {r\delta_i}/({np_i(n+r)})} \right) &\le -r\delta_i/n+\frac{2r^2\delta_i^2}{n^2p_i(n+r)},\nonumber\\
\implies    D_{KL}(v_i\parallel \hat{v}_i) &\le  \frac{2r^2\delta_i^2}{n^2p_i(n+r)},\nonumber\\
\implies \sum_{i=1}^{k'} D_{KL}(v_i\parallel \hat{v}_i) &\le  \frac{2r^2}{n^2}\sum_{i=1}^{k'}\frac{\delta_i^2}{np_i}.\label{eq:discrete1}
\end{align}
We will now bound $\sum_{i=1}^{k'}\frac{\delta_i^2}{np_i}$. As a $\poisson(\lambda)$ random variable has variance $\lambda$ and $\delta_i=u_i-np_i$ where $u_i\leftarrow \poisson(np_i)$, therefore,
\begin{align*}
    \E\left[ \sum_{i=1}^{k'}\frac{\delta_i^2}{np_i} \right] = k'. 
\end{align*}
Also, the fourth central moment of a $\poisson(\lambda)$ random variable is $\lambda(1+3\lambda)$, hence 
\begin{align*}
    \Var[\delta_i^2] &= \E \left [\delta_i^4 \right ]-\E \left [\delta_i^2 \right ]^2,\\
    &= np_i(1+3np_i)-(np_i)^2=np_i(1+2np_i),\\
    \implies  \Var\left[ \sum_{i=1}^{k'}\frac{\delta_i^2}{np_i} \right] &= \sum_{i=1}^{k'} \frac{1+2np_i}{np_i}.
\end{align*}
As $p_i \ge {\epsilon}/{(2nk)}$ for $i\in S$ and $k'\le k$, therefore, 
\begin{align*}
    \Var\left[ \sum_{i=1}^{k'}\frac{\delta_i^2}{np_i} \right]  \le 2k^2/\epsilon+2k\le 4k^2/\epsilon.
\end{align*}
Hence by Chebyshev's inequality,
\begin{align}
    \Pr\left[ \sum_{i=1}^{k'}\frac{\delta_i^2}{np_i}-k' \ge 4k/{\epsilon} \right] &\le \epsilon/4,\nonumber\\
\implies \Pr\left[ \sum_{i=1}^{k'}\frac{\delta_i^2}{np_i} \ge 4k/{\epsilon} \right] &\le \epsilon/4.\label{eq:discrete2}
\end{align}
Let $E_1$ be the event that $\sum_{i=1}^{k'}\frac{\delta_i^2}{np_i} \le 4k/{\epsilon}$. By \eqref{eq:discrete2}, $\Pr(E_1)\ge 1-\eps/4$. Conditioned on the event $E_1$ and using \eqref{eq:discrete1}, we can bound the KL divergence as follows,
\begin{align*}
\DKL[\Big]{\prod_{i\in S}{v_i}}{ \prod_{i\in S}{\hat{v}_i} }= \sum_{i=1}^{k'} D_{KL}({v_i}\parallel {\hat{v}_i}) &\le  \frac{8r^2k}{n^2\epsilon}.
\end{align*}
Hence for $r\le n\epsilon^{1.5}/(4\sqrt{k})$ and conditioned on the event $E_1$,
\begin{align*}
 \DKL[\Big]{\prod_{i\in S}{v_i}}{\prod_{i\in S}{\hat{v}_i}} \le \eps^2/2.
 \end{align*}
Hence using Pinsker's inequality, conditioned on the event $E_1$,
\begin{align*}
d_{TV} \left (\prod_{i\in S}{v_i}, \prod_{i\in S}{\hat{v}_i} \right ) \le \eps/2.
\end{align*}
We will now bound the total variation distance for the variables in the set $S^c$. Let $E_2$ be the event that $u_i=0,\; \forall \; i \in S^c$. Note that as $u_i\sim \poisson(np_i)$ where $p_i < \eps/(2nk)$, $u_i=0$ with probability at least $e^{-{\eps}/(2k)}$, hence $\Pr(E_2)\ge  e^{-{\eps}/2}\ge 1-\eps/2$. We now condition on the event $E_2$. Recall that $v_i=y_i+z_i$, where $z_i\sim \poisson(rp_i)$ and $\hat{v}_i=y_i+\hat{z}_i$, where $\hat{z}_i=0$ conditioned on $E_2$. By a coupling argument on $y_i$, the total variation distance between the distributions of $v_i$ and $\hat{v}_i$ equals the total variation distance between the distributions of $z_i$ and $\hat{z}_i$. As $\hat{z}_i=0$, conditioned on the event $E_2$,
\begin{align*}
d_{TV}({v_i}, {\hat{v}_i}) &= \Pr[z_i\ne 0]=1-e^{-rp_i}\le 1-e^{-r\eps/(2nk)}\\
&\le \frac{r\eps}{2nk}\le \frac{\eps}{2k}, \quad \text{as } r\le n.
\end{align*}
Hence conditioned on $E_2$,
\begin{align*}
d_{TV}\left (\prod_{i\in S^c}{v_i}, \prod_{i\in S^c}{\hat{v}_i} \right )\le \sum_{i=k'+1}^{k} d_{TV} \left ({v_i}, {\hat{v}_i} \right ) \le \eps/2.
\end{align*}
Hence conditioned on the events $E_1$ and $E_2$,
\begin{align*}
d_{TV}\left (\prod_{i=1}^{k}{v_i}, \prod_{i=1}^{k}{\hat{v}_i} \right ) \le  d_{TV} \left (\prod_{i\in S}{v_i}, \prod_{i\in S}{\hat{v}_i} \right ) + d_{TV} \left (\prod_{i\in S^c}{v_i}, \prod_{i\in S^c}{\hat{v}_i} \right ) \le \eps.
\end{align*}
As $\Pr(E_1) \ge 1-\eps/4 $ and $\Pr(E_2) \ge 1-\eps/2$, by a union bound $\Pr(E_1 \cup E_2)\ge 1-\eps $. Hence with probability $1-\eps$ over the randomness in $\{u_i,i\in [k]\}$,
\begin{align*}
d_{TV} \left (\prod_{i=1}^{k}{v_i}, \prod_{i=1}^{k}{\hat{v}_i} \right)  \le \eps.
\end{align*}
\end{proof}
Lemma \ref{lem:discrete1} says that with high probability over the first $N_1$ samples, the $N_2 + R$ samples are close in total variation distance to $\Poisson(n+r)$ samples drawn from $D$. Using lemma \ref{lem:discrete1} and lemma \ref{lem:coupling}, we can conclude that for $r\le n\epsilon^{1.5}/(4\sqrt{k})$, $D_{TV}(X_{M}, Z_{\tilde{M}}) \leq \epsilon + \epsilon/2 = 3\epsilon/2$.

Next, we show how to use the above amplification procedure to amplify samples in the non-Poissonized setting. Given $N = N_1 + N_2$ samples from $D$, we have shown how to amplify them to get ${\tilde{M}} = N+R$ samples. Given such an amplifier as a black box, and $4n$ samples from $D$, one can use the first $N$ samples to generate $M$ samples. Then append these $M$ samples with the remaining $4n - N$ samples to get an amplifier in our original non-Poissonized setting.

\begin{lemma}\label{lem:discrete-ub2}
Let $N = N_1 + N_2$ where $N_1, N_2 \leftarrow \poisson(n)$, and let $M \leftarrow \poisson(2n+r)$.
 Suppose we are given an $(N,M)$ amplifier $f$ (as described above) satisfying $D_{TV}(f(X_N), X_{M}) \leq \frac{3\eps}{2}$, for all $D \in \mathcal C$. Then there exists an amplifier $f': \mathcal [k]^{4n} \rightarrow \mathcal [k]^{4n+\frac{r}{8}}$, such that  $D_{TV}(f'(X_{4n}), X_{4n+\frac{r}{8}}) \leq \frac{5 \epsilon}{2}$, for $\epsilon \geq 2e^{-\frac{n}{20}} + e^{-\frac{25r}{88}}$, and for $r\le n\epsilon^{1.5}/(4\sqrt{k})$.
\end{lemma}

\begin{proof}
We divide the proof into three steps:
\begin{itemize}
    \item \textbf{Step 1: } $f$ takes as input $X_{N_1}$ and $X_{N_2}$, samples of size $N_1$ and $N_2$ drawn from $D$. To simulate these samples, we use the 4n samples available to us from $D$. We draw $N_1', N_2' \leftarrow \poisson(n)$, and let $N' = N_1'+N_2'$. If $N' \leq 4n$, we set $X_{N_1'} = (x_1, x_2, \dots, x_{N_1'})$ and $X_{N_2'} = (x_{N_1'+1},x_{N_1'+2}, \dots, x_{N_2'})$. Otherwise, we set $X_{N_1'}=\underbrace{(x_1, x_1, \dots, x_1)}_{N_1' \text{ times}}$, and $X_{N_2'}=\underbrace{(x_1, x_1, \dots, x_1)}_{N_2' \text{ times}}$, but this happens with very small probability leading to small total variation distance between $f(X_{N_1}, X_{N_2})$ and $f(X_{N_1'}, X_{N_2'})$, and by triangle inequality, small TV distance between $f(X_{N_1'}, X_{N_2'})$ and $X_M$. We denote $(X_{N_1}, X_{N_2})$ by $X_N$ and $(X_{N_1'}, X_{N_2'})$ by $X_{N'}$.
    
    \item \textbf{Step 2: } We would like to finally output $\frac{r}{8}$ more samples. Let us denote the number of samples in $f(X_{N'})$ by $M'$. If $M' < N'+\frac{r}{8}$, we append $N'+\frac{r}{8} - M'$  arbitrary samples to it (say $x_1$) so that the total sample size is equal to $N' + \frac{r}{8}$.  If $M' \geq N' + \frac{r}{8}$, we don't do anything in this step. Let $t_1(f(X_{N'}))$ denote the samples outputted in this step. Since the number of new samples added by $f$ is roughly distributed as $\poisson(r)$, the probability that the number of new samples  is less than $r/8$ is small, leading to small $TV$ distance between $t_1(f(X_{N'}))$ and $f(X_{N'})$, and by triangle inequality, small TV distance between  $t_1(f(X_{N'}))$ and $X_M$.
    \item \textbf{Step 3: } Let $M_1'$ denote the number of samples in $t_1(f(X_{N'}))$, and let $Q_1' = 4n+\frac{r}{8} - M_1'$ denote the number of extra samples needed to output $4n+\frac{r}{8}$ samples in total. If $Q_1' \geq 0$, we append $Q_1'$ i.i.d. samples from $D$ to $t_1(f(X_{N'}))$, and if $Q_1' < 0$, we remove last $\lvert Q_1' \rvert$ samples from $t_1(f(X_{N'}))$. We use  $t_2(t_1(f(X_{N'})))$ to denote the output of this step. Step 2 ensures $M_1' \geq N' +\frac{r}{8}$, which implies $Q_1' \leq 4n - N'$. Let $X_{4n-N'} = (x_{N'+1}, x_{N'+2}, \dots, x_{4n})$ denote the leftover samples in $X_{4n}$ after removing the first $N'$ samples. When $Q_1' \geq 0$, we use the first $Q_1'$ samples from  $X_{4n-N'}$ to simulate i.i.d. samples from $D$, that is, 
    $t_2(t_1(f(X_{N'}))) = \text{append}(t_1(f(X_{N'})), (x_{N'+1}, x_{N'+2}, \dots, x_{N'+Q_1'}))$. $t_2(t_1(f(X_{N'})))$ is the final output of our amplifier $f'$.

    Similarly, let $Q_1 = 4n+\frac{r}{8} - M$ denote the number of extra samples needed to be appended to $X_M$ to output $4n+\frac{r}{8}$ samples in total. If $Q_1 \geq 0$, $t_2(X_M)$ correspond to appending $Q_1$ samples from $D$ to $X_M$, and otherwise, it corresponds to removing last $\lvert Q_1 \rvert$ samples from $X_M$. Since applying the same transformation to two random variables can't increase their total variation distance, and from step 2, we know that $D_{TV}(t_1(f(X_{N'})), X_M)$ is small, we get $D_{TV}(t_2(t_1(f(X_{N'}))), t_2(X_M))$ is small. 
    
    As $t_2(X_M)$ corresponds to $4n + \frac{r}{8}$ i.i.d. samples from $D$,  $D_{TV} (X_{4n+\frac{r}{8}},t_2(X_M) ) = 0$. Using triangle inequality, we get $D_{TV}(t_2(t_1(f(X_{N'}))) ,X_{4n+\frac{r}{8}})$ is small which is the desired result.
\end{itemize}
 Next, we prove that the total variation distances involved in each of these steps are small.

\begin{itemize}
    \item \textbf{Step 1: } We first bound $D_{TV}(f(X_N), f(X_{N'}))$.
    \begin{equation*}
    \begin{split}
        D_{TV}(f(X_N), f(X_{N'})) &\leq D_{TV}(X_{N}, X_{N'})\\
        &=\frac{1}{2}\sum_{x} \ \lvert \Pr(X_{N} = x) - \Pr(X_{N'} = x) \rvert\\
        &=\frac{1}{2}\sum_{x} \ \lvert \Pr(X_{N} = x \mid N \leq 4n)\Pr(N \leq 4n) - \Pr(X_{N'} = x \mid N' \leq 4n)\Pr(N' \leq 4n) \\
        &+ \Pr(X_{N} = x \mid N > 4n)\Pr(N > 4n) -
        \Pr(X_{N'} = x \mid N' > 4n)\Pr(N' > 4n) \rvert
    \end{split}
\end{equation*}
where the first inequality holds as applying the same transformation to two random variables can't increase their total variation distance. Now, note that $X_N$ and $X_{N'}$   have the same distribution conditioned on $N \leq 4n$ and $ N' \leq 4n$. Also, $\Pr(N \leq 4n) = \Pr(N' \leq 4n)$ and $\Pr(N > 4n) = \Pr(N' > 4n)$, as both $N$ and $N'$ are drawn from $\poisson(2n)$ distribution. This gives us
\begin{equation*}
    \begin{split}
        D_{TV}(f(X_N), f(X_{N'}))  
        &=\frac{1}{2}  \sum_{x} \  \Pr(N > 4n) \lvert 
         \Pr(X_{N} = x \mid N > 4n) -
        \Pr(X_{N'} = x \mid N' > 4n) \rvert\\
        &\leq \Pr(N > 4n)
    \end{split}
\end{equation*}
Using the triangle inequality, we get $D_{TV}(X_M, f(X_N')) \leq \Pr(N > 4n) + 3\epsilon/2$. To bound $\Pr(N > 4n)$, we use the following Poisson tail bound \citep{canonne2017short}: for $X \leftarrow \poisson(\lambda)$, 

\begin{equation}\label{eq:poisson-tail}
\Pr[X \geq \lambda + x], \Pr[X \leq \lambda - x] \leq e^{\frac{-x^2}{\lambda + x}}.    
\end{equation}

As $N$ is distributed as $\poisson(2n)$, we get $\Pr(N > 4n) \leq e^{-n}$, which implies $D_{TV}(X_M, f(X_N')) \leq e^{-n} + \frac{3\epsilon}{2}$.
    
    \item \textbf{Step 2: } In this step, we need to show $D_{TV}(t_1(f(X_{N'})), X_M)$ is small. Note that $t_1(f(X_{N'}))$ is equal to $f(X_{N'})$ except when $M' < N'+\frac{r}{8}$. From step 1, we know that $D_{TV}(f(X_{N'}), X_M)$ is small. If we show $D_{TV}(f(X_{N'}), t_1(f(X_{N'})))$ is small, then by triangle inequality, we get $D_{TV}(X_M, t_1(f(X_{N'})))$ is small. Let $M' = N'+R'$ where $R'$ denote the number of new samples added by the amplification procedure $f$ to $X_{N'}$.
\begin{equation*}
    \begin{split}
        &D_{TV}\left(t_1\left(f\left(X_{N'}\right)\right), f\left(X_{N'}\right)\right)\\ 
        &=\frac{1}{2}\sum_{x} \ \lvert \Pr\left(t_1\left(f\left(X_{N'}\right)\right) = x\right) - \Pr\left(f\left(X_{N'}\right) = x\right) \rvert\\
        &= \frac{1}{2}\sum_{x} \ \lvert \Pr\left(R' < \frac{r}{8}\right) \left(\Pr\left(t_1\left(f\left(X_{N'}\right)\right) = x \mid R' < \frac{r}{8}\right)  - \Pr\left(f\left(X_{N'}\right) = x \mid R' < \frac{r}{8}\right)\right)\\
        & + \Pr\left(R' \geq \frac{r}{8}\right)\left(\Pr\left(t_1\left(f\left(X_{N'}\right)\right) = x \mid R' \geq \frac{r}{8}\right)  - \Pr\left(f\left(X_{N'}\right) = x \mid R' \geq \frac{r}{8}\right)\right)\rvert\\
    \end{split}
\end{equation*}
   We know $\Pr\left(t_1\left(f\left(X_{N'}\right)\right) = x \mid R' \geq \frac{r}{8}\right)  = \Pr\left(f\left(X_{N'}\right) = x \mid R' \geq \frac{r}{8}\right)$. This gives 
    \begin{equation*}
    \begin{split}
        &D_{TV}\left(t_1\left(f\left(X_{N'}\right)\right), f\left(X_{N'}\right)\right)\\ 
        &= \frac{1}{2}\sum_{x} \ \lvert \Pr\left(R' < \frac{r}{8}\right) \left(\Pr\left(t_1\left(f\left(X_{N'}\right)\right) = x \mid R' < \frac{r}{8}\right)  - \Pr\left(f\left(X_{N'}\right) = x \mid R' < \frac{r}{8}\right)\right) \rvert\\
        &\leq \Pr\left(R' < \frac{r}{8}\right)
    \end{split}
\end{equation*}
    Now, we need to bound $\Pr\left(R' < \frac{r}{8}\right)$. From the description of $f$, we know that the number of new copies of element $i$ added by $f$ is distributed as $\poisson\left(r \hat{p}_i\right)$. Here, $\hat{p}_i =\frac{u_i}{n}$ where $u_i$ denotes the number of occurrences of element $i$ in $X_{N_1'}$. Since the total number of samples in $X_{N_1'}$ is $N_1'$, we get $\sum_{i=1}^k \hat{p}_i = \frac{\sum_{i=1}^k u_i}{n} = \frac{N_1'}{n}$. Note that $R'$ is equal to the sum of number of new copies of each element, and as the sum of Poisson random variables is Poisson, we get $R'$ is distributed as $\poisson\left(r\frac{N_1'}{n}\right)$. 
    \begin{equation*}
    \begin{split}
        \Pr\left(R' < \frac{r}{8}\right) &= \Pr\left(R' < \frac{r}{8} \mid {N_1'} \geq \frac{3n}{4}\right) \Pr\left({N_1'} \geq \frac{3n}{4}\right) + \Pr\left(R' < \frac{r}{8} \mid {N_1'} < \frac{3n}{4}\right) \Pr\left({N_1'} < \frac{3n}{4}\right) \\
        &\leq \Pr\left(R' < \frac{r}{8} \mid {N_1'} \geq \frac{3n}{4}\right) + \Pr\left({N_1'} < \frac{3n}{4}\right)
    \end{split}
    \end{equation*}
    Using Poisson tail bound (\ref{eq:poisson-tail}), we get 
    \begin{equation*}
        \begin{split}
            \Pr\left(R' < \frac{r}{8} \mid {N_1'} \geq \frac{3n}{4}\right) &\leq exp\left(-\frac{\left(5r/8\right)^2}{3r/4+5r/8}\right) = e^{-25r/88}\\
            \Pr\left(N_1' < \frac{3n}{4}\right) &\leq exp\left(-\frac{\left(n/4\right)^2}{n+n/4}  \right) = e^{-n/20}
        \end{split}
    \end{equation*}
    This gives us $D_{TV}(f(X_{N'}), t_1(f(X_{N'}))) \leq e^{-25r/88} + e^{-n/20}$. By triangle inequality, we get $D_{TV}(X_M, t_1(f(X_{N'}))) \leq \frac{3\epsilon}{2}+  e^{-n}  + e^{-25r/88} + e^{-n/20}$.
    \item \textbf{Step 3: } For this step, we need to show $D_{TV}(t_2(t_1(f(X_{N'}))), t_2(X_M))$ is small. Since applying the same transformation to two random variables doesn't increase their TV distance, we get 
    \begin{equation*}
        \begin{split}
            D_{TV}(t_2(t_1(f(X_{N'}))), t_2(X_M)) &\leq D_{TV}(t_1(f(X_{N'})), X_M)\\
            &\leq \frac{3\epsilon}{2}+  e^{-n} +  + e^{-25r/88} + e^{-n/20}
        \end{split}
    \end{equation*}
    As $D_{TV} (X_{4n+\frac{r}{8}},t_2(X_M) ) = 0$, using triangle inequality, we get 
    \begin{equation*}
        \begin{split}
            D_{TV}(t_2(t_1(f(X_{N'}))),X_{4n+\frac{r}{2}} )
            &\leq \frac{3\epsilon}{2}+  e^{-n} +   e^{-25r/88} + e^{-n/20}
        \end{split}
    \end{equation*}
\end{itemize}
For $\epsilon \geq 2e^{-n/20}+e^{-25r/88}$, this gives us $D_{TV}(f'(X_{4n}),X_{4n+\frac{r}{8}} ) = D_{TV}(t_2(t_1(f(X_{N'}))), X_{4n+\frac{r}{8}} ) \leq \frac{5 \epsilon}{2}$.
\end{proof}
From lemma \ref{lem:discrete-ub2}, we get that for $\epsilon \geq 2e^{-n/20}+e^{-25r/88}$, and for $r \leq n\epsilon^{1.5}/(4\sqrt{k})$, $D_{TV}(f'(X_{4n}), X_{4n+\frac{r}{8}} ) \leq \frac{5 \epsilon}{2}$. We can assume $n$ is at least $\sqrt{k}$, and $r$ is at least $8$, as otherwise the theorem is trivially true. So for $k$ large enough (implying large $n$), we can put $\epsilon = \frac{2}{15}$, to get $D_{TV}(t_2(t_1(f(X_{N'}))), X_{4n+\frac{r}{8}} ) \leq \frac{1}{3}$, which finishes the proof!
\end{proof}

%\subsection{Lower bound for discrete distributions with bounded support}
\subsection{Lower Bound}

In this section we show that the above procedure is optimal, up to constant factors for amplifying samples from discrete distributions. The proof is constructive and shows that a simple verifier can distinguish any amplifier when 
%\replaced
{$m \geq c \frac{n}{\sqrt{k}}$ for some fixed $c$.}
%{$m > \alpha \frac{n}{\sqrt{k}}$ for a fixed $\alpha$.} 
The proof relies on the fact that the amplifier cannot add samples beyond the support of the samples it has already seen. When $m$ is sufficiently larger than $n$, we can show there are distributions for which large regions of the support are below the threshold required for the birthday paradox meaning that with high probability every new sample will reveal additional information about the support. The amplifier will not be able to add samples in that region. 

\begin{proposition}\label{prop:disc_lower}
There is a constant $c$, such that for every sufficiently large $k$,  $\mathcal{C}$ does not admit an $\left(n, n+\frac{cn}{\sqrt{k}}\right)$ amplification procedure.
\end{proposition}
The proposition follows by constructing a verifier and class of discrete distributions over $k$ elements, $\mathcal C$ with the following property: for a universal constant $c$ and $p \leftarrow Uniform[\mathcal C]$, the verifier can detect \emph{any} $(n, n+ \frac{cn}{\sqrt d})$ amplifier with sufficiently high probability.

%\begin{theorem} \label{thm:disc_lower}
%There exists a class $\mathcal C$ of distributions over $k$ elements such that for any amplifier $f$, $p \sim Uniform[\mathcal C^k]$, $m >n+\alpha \frac{n}{\sqrt k}$, for a constant $\alpha$, $X_n \sim p^n$ and $Y_m \sim p^m$ there exists a verifier that can distinguish $f(X_n)$ from $Y_m$.
%\end{theorem}

Before we prove Proposition \ref{prop:disc_lower}, we introduce some additional notation and a basic martingale inequality. Let $C^k$ be the  set of discrete uniform distributions over $k$ integers in $0,\dots,8k$. Let $C^k_l$ be the set of discrete distributions with mass $1-l$ on one element and uniform mass over $k-1$ remaining integers in 
%\replaced
{$0, \dots, 8(k-1) $ }
%{$0, \dots, 8k$}
. We also rely on some martingale inequalities which can be found in \cite{chung2006concentration}.
\begin{fact} \label{fact:martingale}
Let $X$ be the martingale associated with a filter $\mathcal F$ 
%\added
{and a sequence of random variables $X_0, X_1, \ldots ,X_n$ such that $X_i = E[X | \mathcal F_i]$, and in particular, $X_0 = E[X]$ and $X_n = X$. Let the martingale satisfy}
\begin{enumerate}
    \item $\mathrm{Var}[X_i \mid \mathcal F_{i-1}] \le \sigma^2_i$ for $1 \le i \le n$
    \item 
    %\replaced
    {$0 \le | X_i - X_{i-1} | \le 1$ }
    %{$0 \le X_i \le 1$}
    almost surely, for $1 \le i \le n$.
\end{enumerate}
Then, we have
$$ Pr(X - \mathbb E[X] \ge \lambda ) \le e^{-\frac{\lambda^2}{2 \left ( \sum \sigma_i^2 + \lambda/3  \right )}}.$$
Similarly the following holds (though not simultaneously):
$$ Pr(X - \mathbb E[X]  \le -\lambda ) \le e^{-\frac{\lambda^2}{2 \left ( \sum \sigma_i^2 + \lambda/3  \right )}}.$$
\end{fact}

Finally we rely on slight generalization of the birthday paradox which can be found in \cite{bellare1994security}.
\begin{fact} \label{fact:bday}
Let $n$ samples be drawn from a uniform distribution over $k$ elements. Then the probability of the samples containing a duplicate is less than $\frac{n^2}{2k}$.
\end{fact}

The proof proceeds in two parts. First we prove a lemma that shows the desired result for 
%\replaced
{$n \le \frac{k}{4}$ }
%{$n \le \frac{k}{2}$}
. We then show show a class of distributions that allows us to reduce the general case to the result shown in the lemma. 

\begin{lemma} \label{lemma:discrete_base_case}
For sufficiently large $k$, 
%\added
{$n \leq \frac{k}{4}$} 
%\deleted{, fixed $c$}
and 
%\replaced
{$m=n  + 30 \frac{n}{\sqrt k} < \frac{k}{2}$ }
%{$m=n  + 30 \frac{n}{\sqrt k}\le \frac{k}{4}$}
the following holds:  

There exists a verifier that for $p \sim Uniform[C^k]$ the following holds true:
\begin{enumerate}
    \item For all $p \in C^k$, it accepts $X_m$ with probability at least $\frac 3 4$ over the randomness in $X_m$.
    \item It rejects $f(X_n)$ with probability at least $\frac 3 4$ for any amplifier $f$ over the randomness in $X_n,p$ and the amplifier.
\end{enumerate}
\end{lemma}

\begin{proof}
First we consider the case when $n \le \frac{\sqrt k}{2}$. Consider the verifier that takes first 
%\replaced
{$n + 1 < \sqrt{\frac k 2} $ }
%{$\frac{\sqrt k}{2} + 1 < \sqrt{\frac{k}{2}}$} 
samples from the given samples and accepts if there are no repeats  and the support is correct. The probability of a duplicate with the real distribution is less than $\frac 1 4$ by fact \ref{fact:bday} so the verifier will accept samples from the true distribution with at least probability $\frac 3 4$.

An amplified set can add repeats of the elements it already saw in which case it will be immediately rejected, or it can add elements outside of the set of elements it saw in which case also, the verifier would catch it with probability at least $\frac 7 8$.
%This is because if the amplifier expanded the support of the set, the verifier would catch it with probability $\frac 7 8$. 
To show this, consider a sample added by the amplifier outside of the seen support. Conditioned on the at most $\frac k 4$ unique samples seen so far (which implies that $\frac 3 4$ of the support is still unseen), the probability, over the choice of $p$, of said sample being in the set is at most 
\begin{equation}
\label{eq:disc_lower_bound_calc1}
    \frac{(3/4)k}{8k -n} \le \frac{(3/4)k}{7.5k} \le \frac 1 8.
\end{equation} 
Hence if the amplified set has any element outside the seen support then it is rejected with probability at least $\frac{7}{8}$. 
%Note that if the amplified set has repeats, then it is immediately rejected.

We now examine the case when $n > \frac{\sqrt{k}}{2}$. 
%Since the verifier can identify when the amplifier introduces unseen elements with probabiltiy at least $\frac 7 8$, we  condition on the event that the amplifier doesn't add new elements outside of the seen support.
For now, let us assume that the amplified set does not contain any other elements other than those present in the original set $X_n$.
The proof proceeds by showing that the amplified set must have significantly more unique elements than the original set. Before we proceed with the details of the proof we define the martingale that is central to the argument.
Consider the scenario where the $n$ samples are drawn in sequence, and let $\mathcal F_i$ denote the filtration corresponding to the $i$-th draw (i.e., information in the first $i$ draws). Let $U_i$ be the indicator that the $i$th sample was previously unseen. Let $U^n = \sum\limits_{i = 1}^n U_i$. Note that $B_i = \mathbb E \left [\sum\limits^n_{j=1} U_j \mid \mathcal F_{i} \right]$ is a Doob martingale with respect to the filtration $\mathcal F_i$ and $B_n=U^n$. Also, $B_i$ has differences bounded by 1 as $U_i$ is an indicator random variable. If $j$ is the count of previously seen elements then 
%\replaced
{$\mathrm{Var} \left [B_i \mid \mathcal F_{i-1} \right ] \le \mathrm{Var}[U_i \mid \mathcal F_{i-1}] \le \frac{(k-j)j}{k^2} $ }
%{$\mathrm{Var} \left [B_i \mid \mathcal F_{i} \right ] \le \mathrm{Var}[U_i \mid \mathcal F_{i}] \le \frac{(k-j)j}{k^2}$}
. The variance is upper bounded by $\frac i k \le \frac n k$.

The verifier will accept only if all elements are within the support of the distribution and the number of unique elements is greater than $\mathbb E[U^n] + 7 \frac{n}{\sqrt k}$.

The remainder of the proof will show the following:
\begin{enumerate}
    \item $U^n$ concentrates around its expectation within a $O\left ( \frac{n}{\sqrt k} \right)$ margin for $X_n$ (this shows the amplifier gets too few unique samples to be accepted by the verifier).
    \item The expectation $\E[U^m-U^n]$ increases by at least $\Omega \left (  \frac{n}{\sqrt k} \right )$ from $X_n$ to $X_m$ (which shows the number of unique items is sufficiently different in expectation between $X_n$ and $X_m$).
    \item $U^m$ concentrates around its expectation within a $O\left ( \frac{n}{\sqrt k} \right)$ margin for $X_m$ (this combined with the previous statement shows the verifier accepts real samples with sufficiently high probability).
\end{enumerate}

% acceptance requires lower tail bound
The upper tail bound follows via Fact \ref{fact:martingale}. Recall that $\frac{n}{\sqrt k} < 4 \frac{n^2}{k}$ since $n > \frac{\sqrt k}{2} $. 
\begin{align}
\begin{split}
\label{eq:disc_lower_bound_calc2}
    \Pr \left(U^n - \mathbb E[U^n] \ge 7\frac{n}{\sqrt k} \right ) &\le \exp \left ( {-\frac{7^2\frac{n^2}{k}}{2 \left ( \sum \sigma_i^2 + \frac{7 n}{3 \sqrt k}  \right )}} \right )\\
    & \le \exp \left ( {-\frac{7^2\frac{n^2}{k}}{2 \left (  \frac{n^2}{k} + \frac{7 n}{3 \sqrt k}  \right )}} \right )\\
    & \le \exp \left ( {-\frac{7^2\frac{n^2}{k}}{2 \left (  \frac{n^2}{k} + 7\frac{4  n^2}{3 k}  \right )}} \right ) \\
   & = \exp \left ( {-\frac{7^2\frac{n^2}{k}}{2 \left (1+ \frac{4}{3} 7 \right) \frac{n^2}{k}}} \right )\\
%     & = \exp \left ( {-\frac{7^2}{2 \left (1+ \frac{4}{3} 7 \right) }} \right )\\
     & \le \frac 1 8 .
     \end{split}
\end{align} 

This shows that for all $p \in \mathcal{C}^k$, with probability at least $\frac{7}{8}$ over the randomness of $X_n$, the amplifier cannot produce a set with enough distinct elements without including unseen elements to it. For such $X_n$, the amplifier will need to include unseen elements in the output in which case the verifier will catch it with probability at least $\frac{7}{8}$ over the randomness in $p$ (by Equation \ref{eq:disc_lower_bound_calc1}). Together this shows that the verifier will reject the amplified samples with probability at least $\frac{3}{4}$ over the randomness in $X_n$, $p$ and the amplifier.

%Note that this suffices to show that the verifier can distinguish any amplifier with sufficiently many unique samples. 
Next, we need to show that the verifier accepts real samples $X_m$ with probability at least $\frac{3}{4}$.

% Increase in expectation 

Let $k$ be sufficiently large that the following conditions hold for both $k$ and $k-1$:
\begin{enumerate}
    \item $n + 30 \frac{n}{\sqrt k} < \frac{k}{2}$,
    \item The samples increased by at most a factor of $2$, that is, $m \leq 2n$.
\end{enumerate}
Now we note that 
$\mathbb E[U^n] $ and $\mathbb E[U^m] $ must differ by at least $\frac{15 n}{\sqrt k}$, since $m < \frac k 2$ implying that every new sample has at least a $\frac 1 2$ probability of being unique:
\begin{equation}
\label{eq:disc_lower_bound_calc3}
    \mathbb E[U^m] - \mathbb E[U^n] \geq \frac{15 n}{\sqrt k}
\end{equation}

Now all the remains to show that the verifier will accept $X_m$ is to show concentration of $U$ within $\frac{8 n}{\sqrt k}$ of its mean. 

Since the number of samples increased by at most a factor of two, the bound on the $\sigma_i^2$ increased by at most a factor of two. This suffices for the lower tail bound on $U$ for $X_m$:

\begin{align}
\label{eq:disc_lower_bound_calc4}
\begin{split}
    \Pr \left(U^m - \mathbb E[U^m] \le -8\frac{n}{\sqrt k} \right ) &\le \exp \left ( {-\frac{8^2\frac{n^2}{k}}{2 \left ( \sum \sigma_i^2 + \frac{8 n}{3 \sqrt k}  \right )}} \right )\\
    & \le \exp \left ( {-\frac{8^2\frac{n^2}{k}}{2 \left (  4\frac{n^2}{k} + \frac{8 n}{3 \sqrt k}  \right )}} \right )\\
    & \le \exp \left ( {-\frac{8^2\frac{n^2}{k}}{2 \left (  4\frac{n^2}{k} + 8\frac{4  n^2}{3 k}  \right )}} \right ) \\
    & = \exp \left ( {-\frac{8^2\frac{n^2}{k}}{2 \left (4+ \frac{4}{3} 8 \right) \frac{n^2}{k}}} \right )\\
  %   & = \exp \left ( {-\frac{8^2}{2 \left (4+ \frac{4}{3} 8 \right) }} \right ) \\
     & < \frac 1 8 .
     \end{split}
\end{align}
Thus $X_m$ will have sufficiently many unique elements to be accepted by the verifier with  probability at least $\frac 7 8$. %A success probability of $\frac 3 4$ in rejecting the amplified samples follows from subtracting the probabiltiy that the verifier did not properly identify unseen samples. 

\begin{comment}
&\le \exp \left ( {-\frac{Q^2\frac{n^2}{k}}{2 \left ( \sum \sigma_i^2 + \frac{Q n}{3 \sqrt k}  \right )}} \right )\\
    & \le \exp \left ( {-\frac{Q^2\frac{n^2}{k}2}{2 \left ( \left (n + \frac{P n}{\sqrt k} \right ) \frac{n + \frac{P n}{\sqrt k}}{k} + \frac{Q n}{3 \sqrt k}  \right )}} \right )\\
    & \le 
    \end{comment}

%Consider the case where $m < \frac 3 4 k$. Then \todo{handle %this assumption}
% 7 n / sqrt k from the first, and 7 (n / sqrt k)
% rejection requires upper tail bound in input
%Now we show the lower tail bound for $X_m$. Recall that $k \ge 2$
%\begin{align*}
%    Pr \left(U - \mathbb E[U] \ge Q\frac{n}{\sqrt k} \right ) &\le \exp \left ( {-\frac{Q^2\frac{n^2}{k}}{2 \left ( \sum \sigma_i^2 + \frac{Q n}{3 \sqrt k}  \right )}} \right )\\
    %& \le \exp \left ( {-\frac{Q^2\frac{n^2}{k}2}{2 \left ( %\left (n + \frac{P n}{\sqrt k} \right ) \frac{n + \frac{P n}{\sqrt k}}{k} + \frac{Q n}{3 \sqrt k}  \right )}} \right )\\
%    & \le 
%\end{align*}
% martingale stuff
\end{proof}

We are now ready to prove Proposition \ref{prop:disc_lower}. We will prove the proposition with $c = 800$.

\begin{proof}
If $n \le \frac k 4$, then Lemma \ref{lemma:discrete_base_case} applies directly. In particular, for $p \in \mathcal C^k$ our verifier in this case will correspond to taking the first $n + 30\frac{n}{\sqrt{k}}$ from the given $n + 800\frac{n}{\sqrt{k}}$, and applying the verifier from Lemma \ref{lemma:discrete_base_case} to them. For $p \notin C^k$, we use the trivial verifier that always accepts.

If $n > \frac{k}{4}$, we use the set of distributions 
%\replaced
{$\mathcal C^k_{\frac {k} {8n}} $ }
%{$\mathcal C^k_{\frac {k} {4}}$}
with the intention of applying the verifier from Lemma \ref{lemma:discrete_base_case} on samples that land in the uniform region. For $p \notin C^k_{\frac {k} {8n}}$, we use the trivial verifier that always accepts.

%\added
{In particular, we show that for $n > \frac{k}{4}$, $m = n + 800 \frac{n}{\sqrt{k}}$, $p \sim Uniform[C^k_{\frac {k} {8n}}]$, there exists a verifier such that the following holds: (i) for all $p \in C^k_{\frac {k} {8n}}$, the verifier accepts $X_m$ with probability at least $\frac{3}{4}$ over the randomness in $X_m$, (ii) the verifier rejects $f(X_n)$ with probability at least $0.7$ for any amplifier $f$ over the randomness in $X_n$, $p$ and the amplifier.}

%\replaced
{
The verifier checks that the samples are within the support of the distribution. %and more than $\frac{k}{8} + 70 \sqrt{k} $  samples are in the uniform region. Finally,
It also applies the verifier from Lemma \ref{lemma:discrete_base_case}  on the uniform region, checking that the number of unique samples in the uniform region is at least $\mathbb E[U^{n'}] + 7 \frac{n'}{\sqrt {k-1}}$, for $n' = \frac{k}{8} + \sqrt{2k}$ . Here, $U^{n'}$ is a random variable denoting the number of unique samples seen when we draw $n'$ samples i.i.d. from a uniform distribution over $k-1$ elements.}
%{The verifier will check that the samples are within the support of the distribution, more than $n+7 \frac{n}{\sqrt k}$ samples are in the uniform region and the verifier from Lemma \ref{lemma:discrete_base_case} accepts on the uniform region.} 
%$\frac{k}{8} + \sqrt{2k} + \frac{7}{\sqrt{k-1}} \left( \frac{k}{8} + \sqrt{2k} \right)$}

First note that after $n$ samples, at 
%\added
{least $\frac{k}{8} - \sqrt{2k} $ } and at most 
%\replaced
{$\frac{k}{8} + \sqrt{2k} $ }
%{$\frac{k}{4}+ \frac{\sqrt k}{4}$}
samples will be in the uniform region with probability at least $\frac{15}{16}$ by a Chebyshev bound. 
%Conditioned on this event,  any amplifier that outputs \replaced{$\frac{k}{8} + 70 \sqrt{k} \geq \frac{k}{8} + \sqrt{2k} + \frac{30}{\sqrt{k-1}} \left( \frac{k}{8} + \sqrt{2k} \right)$  }{$\frac{k}{4} + O(\sqrt k)$} (assuming sufficiently large $k$) samples   in the uniform region will be rejected by the verifier  with probability at least $\frac{3}{4}$, 
Conditioned on this event, the amplified samples will be rejected by the verifier with probability at least $\frac{3}{4}$,
for not having enough unique samples in the uniform region or having samples outside the original support. This can be shown by calculations similar to Equations \ref{eq:disc_lower_bound_calc1} and \ref{eq:disc_lower_bound_calc2}, assuming sufficiently large $k$. 
%\added
{Taking the conditioning into account, this shows that 
%any amplifier that outputs $\frac{k}{8} + 70 \sqrt{k}$ samples in the uniform region
amplified samples
will be rejected by the verifier with probability at least $\frac{15}{16}*\frac{3}{4} \geq 0.7$.}

Next we show that the verifier will accept real samples with good probability. Note that the expected number of samples to receive in the uniform region for $X_m$ is 
%\replaced
{$\frac{k}{8} + 100\sqrt{k} $ }
%{$\frac{k}{4} + c \sqrt k$}
. The variance on this quantity is at most 
%\replaced
{$\frac{k}{8} + 100\sqrt{k} $ }
%{$\frac{k}{4} + c \sqrt k$}
. 
%\added
{
An application of Chebyshev's inequality shows that for sufficiently large $k$, with probability at least $\frac {15}{16}$, at least $\frac{k}{8} + 70 \sqrt{k} \geq \frac{k}{8} + \sqrt{2k} + \frac{30}{\sqrt{k-1}} \left( \frac{k}{8} + \sqrt{2k} \right)$ samples will land in the uniform region.  Conditioned on this,  the number of unique samples in the uniform region will be large enough (at least $\mathbb E[U^{n'}] + 7 \frac{n'}{\sqrt {k-1}}$) with probability at least $\frac{7}{8}$. This can be shown by calculations similar to Equations \ref{eq:disc_lower_bound_calc3} and \ref{eq:disc_lower_bound_calc4}, assuming sufficiently large $k$. Taking the conditioning into account, this shows that the verifier will accept real samples with probability at least $\frac{15}{16}*\frac{7}{8} \geq \frac{3}{4}$.}

% \begin{align*}
%     \frac k 4 + c \sqrt k - 4 \sqrt{\frac k 4 + c \sqrt k} &\ge \frac k 4 + c \sqrt k - 2\sqrt{k} -4 \sqrt{c \sqrt k}\\
%     & \ge \frac k 4 + c \sqrt k - 2\sqrt{k} -4 \sqrt{c k}.
% \end{align*}
% Since the expression above is increasing with $c$, we can choose a $c$ sufficiently large so that the verifier will accept with sufficiently high probability. 
\end{proof}

\bibliographystyle{plain}
\bibliography{refs}

\end{document}